\documentclass[letterpaper,11pt]{article}
\usepackage{etex}
\usepackage[margin=1in]{geometry}

\usepackage[dvipsnames,table,xcdraw]{xcolor}
\usepackage{amssymb,amsfonts,amsmath,amstext,amsthm,mathrsfs}
\usepackage{graphics,latexsym,epsfig,psfrag,wrapfig,comment,paralist}
\usepackage{xspace}
\usepackage{subcaption,bm,multirow}
\usepackage{booktabs}
\usepackage{mathdots}
\usepackage{bbold}
\usepackage{centernot}
\usepackage{algorithm}
\usepackage{algorithmicx}
\usepackage[noend]{algpseudocode}
\usepackage{prettyref}
\usepackage{listings}
\usepackage{tikz}
\usetikzlibrary{decorations,calligraphy}
\usepackage{pgfplots}
\usetikzlibrary{decorations.pathreplacing}
\usetikzlibrary{shapes}
\usetikzlibrary{positioning, arrows, matrix, calc, patterns, fit}
\usepackage{caption}
\usepackage{array}
\usepackage{mdwmath}
\usepackage{multirow}
\usepackage{mdwtab}
\usepackage{eqparbox}
\usepackage{multicol}
\usepackage{amsfonts}
\usepackage{multirow,bigstrut,threeparttable}
\usepackage{amsthm}
\usepackage{array}
\usepackage{bbm}
\usepackage{epstopdf}
\usepackage{mdwmath}
\usepackage{mdwtab}
\usepackage{eqparbox}
\usepackage{tikz}
\usepackage{latexsym}
\usepackage{amssymb}
\usepackage{bm}
\usepackage{graphicx}
\usepackage{mathrsfs}
\usepackage{epsfig}
\usepackage{psfrag}
\usepackage{setspace}
\definecolor{linkColor}{HTML}{E74C3C}
\definecolor{pearcomp}{HTML}{B97E29}
\definecolor{citeColor}{HTML}{2980B9}
\definecolor{urlColor}{HTML}{1D2DEC}
\definecolor{conjColor}{HTML}{9ab569}
\usepackage[CJKbookmarks=true,
            bookmarksnumbered=true,
            bookmarksopen=true,
            colorlinks=true,
            citecolor=citeColor,
            linkcolor=linkColor,
            anchorcolor=red,
            urlcolor=urlColor,
            ]{hyperref}
\usepackage{natbib}
\setcitestyle{authoryear,round}
\newtheoremstyle{break}
  {\topsep}{\topsep}%
  {\itshape}{}%
  {\bfseries}{}%
  {\newline}{}%

\usepackage{pgfplots}
\usepackage{amsmath,amssymb}
\usepackage{mathtools}
\usepackage{stmaryrd}
\pgfmathdeclarefunction{gauss}{3}{%
  \pgfmathparse{#3/(#2*sqrt(2*pi))*exp(-((x-#1)^2)/(2*#2^2))}%
}
\pgfmathdeclarefunction{mingauss}{4}{%
  \pgfmathparse{#4/(#3*sqrt(2*pi))*exp(-max((x-#1)^2, (x-#2)^2)/(2*#3^2))}%
}

\tikzset{
  invisible/.style={opacity=0},
  visible on/.style={alt={#1{}{invisible}}},
  alt/.code args={<#1>#2#3}{%
    \alt<#1>{\pgfkeysalso{#2}}{\pgfkeysalso{#3}}
  },
}

\newtheorem{definition}{\textbf{Definition}}

\newtheorem{lemma}{\textbf{Lemma}}
\newtheorem{theorem}{\textbf{Theorem}}

\newtheorem*{insight*}{\textbf{Observation}}
\newtheorem*{theoremi}{\textbf Theorem (informal)}

\newtheorem{prop}{\textbf{Proposition}}

\newtheorem*{lemmai*}{\textbf{Lemma (informal)}}

\newcommand{\cX}{\mathcal{X}}
\newcommand{\cY}{\mathcal{Y}}
\newcommand{\cZ}{\mathcal{Z}}
\newcommand{\cU}{\mathcal{U}}
\newcommand{\cV}{\mathcal{V}}

\newcommand{\cW}{\mathcal{W}}
\newcommand{\cS}{\mathcal{S}}

\newcommand{\cP}{\mathcal{P}}

\newcommand{\cA}{\mathcal{A}}

\newcommand{\cE}{\mathcal{E}}

\newcommand{\cD}{\mathcal{D}}

\usepackage[normalem]{ulem}

\renewcommand{\cite}[1]{\citep{#1}}

\usepackage{color}
\definecolor{cm}{RGB}{0,0,200}
\definecolor{purple}{RGB}{200,0,200}

\def\clip{\operatorname{clip}}

\usepackage[intoc]{nomencl}
\makenomenclature

\makeatletter
\newcommand{\vast}{\bBigg@{2.5}}
\newcommand{\Vast}{\bBigg@{5}}
\makeatother

\DeclareMathOperator{\E}{\mathbb{E}}

\DeclareMathOperator*{\argmax}{arg\,max}
\DeclareMathOperator*{\argmin}{arg\,min}

\usepackage{bbm}

\usepackage[utf8]{inputenc}
\usepackage[LGR,T1]{fontenc}

\usepackage{breqn}

\usepackage{enumitem,kantlipsum}
\newcommand\blfootnote[1]{%
  \begingroup
  \renewcommand\thefootnote{}\footnote{#1}%
  \addtocounter{footnote}{-1}%
  \endgroup
}
\title{Optimal Conservative Offline RL with General Function Approximation via Augmented Lagrangian}
\author{Paria Rashidinejad$^\dagger$ \quad
Hanlin Zhu$^\dagger$ \quad
Kunhe Yang$^{\dagger}$\quad
Stuart Russell$^{\dagger}$ \quad
Jiantao Jiao$^{\dagger, \ddagger}$ 
 \blfootnote{Emails: \texttt{\{paria.rashidinejad,hanlinzhu,kunheyang,russell,jiantao\}@berkeley.edu}}\\ { }\\
$^\dagger$Department of Electrical Engineering and Computer Sciences, UC Berkeley\\
$^\ddagger$Department of Statistics, UC Berkeley\\ { } \\
}
\date{\today}

\begin{document}

\maketitle

\begin{abstract}
Offline reinforcement learning (RL), which refers to decision-making from a previously-collected dataset of interactions, has received significant attention over the past years. Much effort has focused on improving offline RL practicality by addressing the prevalent issue of partial data coverage through various forms of conservative policy learning. While the majority of algorithms do not have finite-sample guarantees, several provable conservative offline RL algorithms are designed and analyzed within the single-policy concentrability framework that handles partial coverage. Yet, in the nonlinear function approximation setting where confidence intervals are difficult to obtain, existing provable algorithms suffer from computational intractability, prohibitively strong assumptions, and suboptimal statistical rates. In this paper, we leverage the marginalized importance sampling (MIS) formulation of RL and present the first set of offline RL algorithms that are statistically optimal and practical under general function approximation and single-policy concentrability, bypassing the need for uncertainty quantification. We identify that the key to successfully solving the sample-based approximation of the MIS problem is ensuring that certain \textit{occupancy validity constraints} are nearly satisfied. We enforce these constraints by a novel application of the augmented Lagrangian method and prove the following result: with the MIS formulation, augmented Lagrangian is enough for statistically optimal offline RL. In stark contrast to prior algorithms that induce additional conservatism through methods such as behavior regularization, our approach provably eliminates this need and reinterprets regularizers as ``enforcers of occupancy validity'' than ``promoters of conservatism.''
\end{abstract}

{
  \hypersetup{linkcolor=black}
  \tableofcontents
}

\section{Introduction}

The goal of offline RL is to design agents that learn to achieve competence in a task using only a previously-collected dataset of interactions \cite{lange2012batch}. Offline RL is a promising tool for many critical applications, from healthcare to autonomous driving to scientific discovery, where the online mode of learning by interacting with the environment is dangerous, impractical, costly, or even impossible \cite{levine2020offline}. Despite this, offline RL has not yet been truly successful in practice \cite{fujimoto2019off, levine2020offline} and impressive RL performance has been limited to settings with known environments \cite{silver2017mastering, moravvcik2017deepstack}, access to accurate simulators \cite{mnih2015human, degrave2022magnetic,fawzi2022discovering}, or expert demonstrations \cite{vinyals2017starcraft}.

One of the central challenges in offline RL is the lack of uniform coverage in real datasets and the \textit{distribution shift} between the occupancy of candidate policies and offline data distribution, which pose difficulties in accurately evaluating the candidate policies. Over the past years, a body of literature has focused on addressing this challenge through developing {conservative} algorithms, which aim at picking a policy among those well-covered in the data. On the practical front, various forms of conservatism are proposed such as behavior regularization through policy constraints \cite{kumar2019stabilizing, fujimoto2019off, nachum2020reinforcement}, learning conservative values \cite{kumar2020conservative, liu2020provably, agarwal2020optimistic}, or learning pessimistic models \cite{kidambi2020morel, yu2020mopo, yu2021combo}; see Appendix \ref{app:related-work} for further discussion on related work.

From a theoretical standpoint, partial data coverage has recently been studied within variants of the single-policy concentrability framework \cite{rashidinejad2021bridging, xie2021bellman, uehara2021pessimistic, song2022hybrid}, which characterizes the distribution shift between offline data and occupancy of a target (often optimal) policy, in contrast to all-policy concentrability commonly used in earlier works \cite{scherrer2014approximate,chen2019information,liao2020batch,zhang2020variational, xie2020batch}. Within this framework and in the tabular and linear function approximation settings, pessimistic algorithms that leverage uncertainty quantifiers to construct lower confidence bounds \cite{jin2020pessimism, rashidinejad2021bridging, yin2021near, shi2022pessimistic, li2022settling} enjoy optimal statistical rate. In the general function approximation setting, pessimistic algorithms largely assume oracle access to uncertainty quantification, either for constructing penalties that are subtracted from rewards \cite{jin2020pessimism, jiang2020minimax} or selecting the most pessimistic option among those that fall within the confidence region implied by the offline data \cite{uehara2021pessimistic, xie2021bellman, chen2022offline}. However, uncertainty quantifiers are difficult to obtain when non-linear function approximation are used and existing heuristics are empirically observed to be unreliable \cite{rashid2019optimistic, tennenholtz2021latent, yu2021combo}. Recent works by \citet{pmlr-v162-cheng22b} and \citet{zhan2022offline} propose provable alternatives to uncertainty-based methods, but leave achieving the optimal statistical rate of $1/\sqrt{N}$, where $N$ is the size of the offline dataset, as an open problem.

Among all, the marginal importance sampling (MIS) methods, which aim at learning weights that estimate the distribution shift between induced policy occupancy and data distribution, lend themselves well to the single-policy concentrability framework. Though more popular in off-policy evaluation \cite{liu2018breaking, xie2019towards, uehara2020minimax, Zhang2020GenDICE:}, MIS has also been used for conservative offline RL such as in AlgaeDICE \cite{nachum2019algaedice} and OptiDICE \cite{lee2021optidice} algorithms, both of which incorporate behavior regularization. Recently, \citet{zhan2022offline} theoretically studied a variant of OptiDICE, showing that MIS with behavior regularization enjoys finite-sample guarantees (though achieving a suboptimal $1/N^{1/6}$ rate) and circumvents certain fundamental difficulties observed in value-based offline RL with function approximation \cite{du2019good, wang2020statistical, wang2021exponential, weisz2021exponential, zanette2021exponential, foster2021offline}.

\subsection{Contributions and results}
Motivated by the benefits offered by MIS, we study designing statistically optimal offline learning algorithms under the MIS formulation and in the general function approximation and single-policy concentrability setting. We conduct theoretical investigations and design algorithms starting from multi-armed bandits (MABs), going forward to contextual bandits (CBs), and finally Markov decision processes (MDPs). We measure the performance of different algorithms by characterizing the {value suboptimality} of the learned policy with respect to the optimal policy. In the rest of this section, we present a preview of our contributions and results.

\paragraph{Multi-armed bandits.}
Empirical MIS algorithms often incorporate behavior regularization, whose role is justified as promoting conservatism by keeping the occupancies of learned and behavior policies close \cite{nachum2019algaedice, lee2021optidice}. Yet, whether and why these regularizers are necessary from a theoretical perspective remain unclear. \citet{zhan2022offline} motivates behavior regularization as a way of introducing curvature in an otherwise linear optimization problem. We extensively investigate the effect of regularization, starting from the simplest setting of MABs with function approximation, as existing algorithms when specialized to offline MABs, are either intractable, have suboptimal finite-sample guarantees, or require access to uncertainty quantifiers.

The following informal theorem states our results on unregularized MIS and MIS with behavior regularization (PRO-MAB Algorithm \ref{alg:PAL-MAB}), which is a special case of PRO-RL algorithm of \citet{zhan2022offline}. Formal statements can be found in Propositions \ref{prop:MAB-alpha-0-fails} and \ref{prop:MAB-constraint-sufficient} and Theorem \ref{thm:suboptimality-MAB}.

\begin{theoremi} The following statements hold for offline learning in MABs with function approximation and single-policy concentrability.
\begin{enumerate}[label=(\Roman*)]
    \item There exists a multi-armed bandit instance where the unregularized MIS fails to achieve a suboptimality that decays with the sample size $N$.
    \item  MIS with behavior regularization (PRO-MAB Algorithm \ref{alg:PAL-MAB}) achieves suboptimality of ${O}(1/\sqrt{N})$.
    \item If one searches only over the space of importance weights that induce policy occupancies that are valid probability distributions, then unregularized MIS achieves $O(1/\sqrt{N})$ suboptimality.
\end{enumerate}
\end{theoremi}
The first part of the above informal theorem states that MIS, in its original form and without any regularization, fails even in multi-armed bandits. For the second part, we conduct a tight analysis of MIS with behavior regularization, demonstrating an optimal rate of $1/\sqrt{N}$ and improving over the $1/N^{1/6}$ rate shown by \citet{zhan2022offline}. In our analysis of PRO-MAB, we find that the key to the success of the regularized MIS algorithm is \textit{near-validity of the learned occupancy}. In the MAB setting, the validity constraint simply requires the learned occupancy to be a probability distribution, i.e., $d_w \coloneqq \sum_a w(a) \mu(a) = 1$, where $a$ represents an arm (action), $w$ denote importance weights, and $\mu$ is the data distribution. With a proper choice of regularization hyperparameter, we show that behavior regularization enforces learned occupancy to be nearly valid: $d_w = \Omega (1)$, that is the normalization factor $d_w$ is at least a positive constant. Finally, we prove that if occupancy validity constraint is satisfied (for example by eliminating all candidate weights $w$ where $\sum_a w(a) \mu(a) \neq 1$), then the unregularized MIS algorithm enjoys optimal sample complexity.

Given the fact that the occupancy validity is the constraint of the optimization problem solved by MIS (see \eqref{eq:MAB-constrained-optimization-with-d}), we ask whether there are any methods for solving an empirical optimization problem that find more constraint-adhering solutions compared to those yielded by the Lagrange multiplier method adopted in prior works \cite{lee2021optidice, zhan2022offline}. The augmented Lagrangian method (ALM), which adds a quadratic loss on the constraints, is a natural choice for our purpose. The ALM term can be easily estimated from offline data and forms our algorithm conservative offline MAB with augmented Lagrangian (Algorithm \ref{alg:PAL-MAB}). We show that the ALM results in $d_w = \Omega(1)$, ensuring near-validity of estimated occupancy and leading to the following guarantee. The formal statement is given in Theorem \ref{thm:suboptimality-PAL-MAB}.

\begin{theoremi}
    The policy returned by conservative offline MAB with augmented Lagrangian (Algorithm \ref{alg:PAL-MAB}) achieves $O(1/\sqrt{N})$ suboptimality.
\end{theoremi}

Our algorithm offers benefits over PRO-MAB such as eliminating the need for choosing the regularization function and only requiring single-policy concentrability instead of the two-policy requirement of PRO-MAB, which can be strong as we discuss in an example in Section \ref{sec:example}. Additionally, behavior regularization introduces bias in the solution even with infinite data \cite{chen2022offline} and the bias-variance tradeoff must be carefully handled. On the other hand, ALM merely enforces the optimization constraints and leads to provably unbiased solutions (Lemma \ref{lemma:optimal-solution-with-AL-unchanged}). More importantly, as we discuss shortly, going beyond {the single-state MAB setting}, convergence rate of behavior regularization becomes suboptimal while ALM maintains the optimal rate.

\paragraph{Contextual bandits.} {In offline CBs, we analyze two approaches: MIS with behavior regularization and an extension of our ALM-based algorithm. We state our results in the following informal theorem, with formal statements given in Proposition \ref{prop:PRO-CB-suboptimal} and Theorem \ref{thm:PAL-CB-suboptimality}.}

\begin{theoremi} The following statements hold for offline learning in contextual bandits with function approximation and single-policy concentrability.
\begin{enumerate}[label=(\Roman*)]
    \item There exists a CB instance where MIS with behavior regularization (PRO-CB Algorithm \ref{alg:PRO-CB}) suffers from suboptimality $\Omega(N^\beta)$ with $\beta > -1/2$, regardless of the choice of regularization hyperparameter.
    \item The policy returned by conservative offline CB with augmented Lagrangian (Algorithm \ref{alg:PAL-CB}) achieves suboptimality of $O(1/\sqrt{N})$.
\end{enumerate}
\end{theoremi}
 
Informally, the failure of PRO-CB to achieve the optimal rate is because the regularization parameter has to be small to control bias, but such small regularization is not strong enough to ensure the validity of learned occupancy in important states. Therefore, one must choose larger regularization, leading to an overall suboptimal rate. Prior works \citet{chen2022offline, pmlr-v162-cheng22b} also allude to this phenomenon, explaining that regularizers appear to be the culprit behind suboptimal rates. In CB, the occupancy validity constraints require conditional occupancy to be a valid probability distribution in most states. In Algorithm \ref{alg:PAL-CB}, we incorporate ALM in offline CBs by adding a weighted sum of quadratic losses describing the validity constraint in each state, where the weights are set to the state occupancies to capture their relative importance. Enforcement of the constraints by the ALM yields the above guarantee on our algorithm.

\sloppy \paragraph{MDPs.} Validity constraints in MDPs ensure that the learned state occupancy $d_w(s) \coloneqq \sum_a w(s,a) \mu(s,a)$ is close to the actual state occupancy $d^{\pi_w}(s)$, where $\pi_w$ is the policy computed from weights $w$.\footnote{One can check that the validity constraints in MAB and CB are special cases of this constraint.} Unlike MABs and CBs, directly enforcing this constraint in MDPs results in an ALM term that cannot be easily estimated from offline data. We address this difficulty by expressing the ALM term in a variational form. From there, we derive model-based and one model-free variants an algorithm called conservative offline RL with augmented Lagrangian (CORAL), that enjoys the following performance upper bound.

\begin{theoremi}
Model-based and model-free CORAL both achieve $O(1/\sqrt{N})$ suboptimality in solving offline RL with general function approximation and single-policy concentrability. 
\end{theoremi}

The formal statement of the above theorem is provided in Theorem \ref{thm:CORAL-suboptimality}. This marks CORAL as the first practical and statistically optimal offline RL algorithm that operates in the general function approximation and partial data coverage setting, while avoiding uncertainty quantification and additional regularizers. Conservatism of CORAL is baked into the MIS formulation and further supported by the ALM: bounded MIS weights prevent the learned occupancy to deviate significantly from data distribution, and the ALM ensures closeness of the learned and actual occupancies. Thus when combined, CORAL learns a policy whose actual occupancy is close to the data distribution.

We proved that the ALM when combined with MIS improves sample complexity compared to alternatives such as behavior regularization. This is in addition to the benefits on optimization stability that are likely to be offered by the ALM, as the ALM improves over the ill-posed Lagrange multiplier objective \cite{ben2022lecture}. Our theoretical findings can explain the empirical observations of \citet{yang2020off}, who find MIS with behavior regularization to be unstable and propose regularizers in ``the spirit of ALM'' that gain superior performance and attribute performance gain to improved optimization. In this work, we present a theoretically-founded way of introducing ALM in offline RL and our analysis shows that ALM also leads to optimal sample complexity.

\section{Background}

\paragraph{Markov decision process.} An infinite-horizon discounted MDP is described by a tuple $M = (\cS, \cA, P, R, \rho, \gamma)$, where $\cS$ is the state space, $\cA$ is the action space, $P: \cS \times \cA \mapsto \Delta(\cS)$ is the transition kernel, $R: \cS \times \cA \mapsto \Delta([0,1])$ encodes a family of reward distributions with $r: \cS \times \cA \mapsto [0,1]$ as the expected reward function, $\rho: \cS \mapsto \Delta(\cS)$ is the initial state distribution, and $\gamma \in [0,1)$ is the discount factor. We assume $\cS$ and $\cA$ are finite however, our results do not depend on their cardinalities and can be naturally extended to infinite sets. A stationary (stochastic) policy $\pi: \cS \mapsto \Delta(\cA)$ specifies a distribution over actions in each state. Each policy $\pi$ induces an occupancy density over state-action pairs $d^\pi: \cS \times \cA \mapsto [0,1]$ defined as $d^\pi(s,a) \coloneqq (1-\gamma)\sum_{t=0}^\infty \gamma^t P_t(s_t = s, a_t = a; \pi)$,
where $P_t(s_t = s, a_t = a; \pi)$ denotes visitation probability of state-action pair $(s,a)$ at step $t$, starting at $s_0 \sim \rho(\cdot)$ and following $\pi$. We     also write $d^\pi(s) = \sum_{a \in \cA} d^\pi(s,a)$ to denote the discounted state occupancy. Additionally, operator $\mathbb{P}^\pi$ is applied to any function $u: \cS \times \cA \rightarrow \mathbb{R}$ and is defined as $(\mathbb{P}^\pi u)(s,a) \coloneqq \sum_{s',a'} P(s'| s,a) \pi(a'| s') u(s',a')$. 

An important quantity is the value a policy $\pi$, which is the discounted sum of rewards $V^\pi(s) \coloneqq \E \left[\sum_{t=0}^\infty \gamma^t r_t \; | \; s_0 = s , a_t \sim \pi(\cdot |s_t) \; \forall \; t\geq 0 \right]$ starting at state $s \in \cS$. Similarly, one can define the Q-function of policy as $Q^\pi(s,a) \coloneqq \E \left[\sum_{t=0}^\infty \gamma^t r_t \; | \; s_0 = s, a_0 = a, a_t \sim \pi(\cdot |s_t) \; \forall \; t\geq 1 \right]$. We use the notation $J(\pi) \coloneqq (1-\gamma) \E_{s\sim \rho}[V^\pi(s)] = \E_{s,a \sim d^\pi} [r(s,a)]$ to represent a scalar summary of the performance of a policy $\pi$.
We denote by $\pi^\star$ the optimal policy that maximizes the above objective and use the shorthand $V^\star \coloneqq V^{\pi^\star}$ to denote the optimal value function. 

\paragraph{Offline reinforcement learning.} We focus on the offline RL, where the agent is only provided with a previously-collected {offline dataset} $\cD = \{(s_i,a_i,r_i,s_i')\}_{i=1}^N$. Here, $r_i \sim R(s_i,a_i)$, $s_i' \sim P(\cdot\mid s_i,a_i)$, and we assume $s_i,a_i$ pairs are generated i.i.d.~according to a data distribution $\mu \in \Delta(\cS \times \cA)$. To streamline the analysis, we assume that the conditional distribution $\mu(a|s)$ is known.\footnote{When $\mu(a|s)$ is unknown, behavioral cloning can be used \cite{ross2014reinforcement, zhan2022offline}.} The goal of offline RL is to learn a policy $\hat{\pi}$ based on the offline dataset so as to minimize the sub-optimality with respect to the optimal policy $\pi^\star$, i.e., $J(\pi^\star) - J(\hat{\pi})$, with high probability. 

\paragraph{Marginalized importance sampling.} In this paper, we consider the marginal importance sampling (MIS) formulation that aims at learning importance weights $w(s,a)$ so as to represent state-action occupancy when multiplied by the offline data distribution. We adopt the following notation:
\begin{align}\label{eq:definitions_dw_sa}
    d_w(s,a) \coloneqq w(s,a) \mu(s,a), \quad d_w(s) \coloneqq \sum_{a \in \cA} d_w(s,a).
\end{align}
We define the policy induced by weights $w$ as 
\begin{align}
\pi_w(a|s) = 
    \begin{dcases}
        \frac{d_w(s,a)}{d_w(s)} = \frac{w(s,a) \mu(a|s)}{\sum_{a \in \cA} w(s,a) \mu(a|s)} \quad & d_w(s) > 0\\
        \frac{1}{|\cA|} \quad & d_w(s) = 0
    \end{dcases}
\end{align}

\paragraph{Offline data coverage assumption.} We design and analyze our algorithms within the single-policy concentrability framework \cite{rashidinejad2021bridging}, stated below.

\begin{definition}[Single-policy concentrability]\label{assump:concentrability} Given a policy $\pi$, define $C^\pi$ to be the smallest constant that satisfies $\frac{d^{\pi}(s,a)}{\mu(s,a)} \leq C^\pi$ for all $s \in \cS$ and $a \in \cA$. 
\end{definition}
$C^\star \coloneqq C^{\pi^\star}$ captures the coverage of $\pi^\star$ in the offline data and is much weaker than the widely used all-policy concentrability that assumes bounded $\max_\pi C^\pi$ \cite{scherrer2014approximate}. In Appendix \ref{app:related-work}, we present a detailed discussion on different coverage assumptions used in offline RL.

\paragraph{Notation.} Given a set $\cS$, we write $|\cS|$ to represent its cardinality and $\Delta(\cS)$ to denote the probability simplex over $\cS$. For a function class $\cW$, we write $|\cW|$ to denote its cardinality (discrete) or covering number (continuous). We use the notation $x \lesssim y$ when there exists a constant $c>0$ such that $x \leq c y$ and $x \asymp y$ if constants $c_1, c_2 > 0$ exist such that $c_1 |x| \leq |y| \leq c_2 |x|$. We write $f(x) = O(g(x))$ if $M > 0, x_0$ exist such that $|f(x)|\leq M g(x)$ for all $x\geq x_0$. 
Define $\clip(x, a, b) \triangleq \max\{ a, \min\{x, b\} \}$ for $x, a, b \in \mathbb{R}$. We write $f(x) = \Omega(g(x))$ if there exists some positive real number $M$ and some $x_0$ such that $|f(x)|\geq  M g(x)$ for all $x\geq x_0$.

\section{Multi-armed bandits}

We start by considering the offline learning problem in the multi-armed bandit (MAB) setting, which is a special case of MDPs with $\gamma = 0$, $|\cS| = 1$, and $ \cD = \{(a_i, r_i)\}_{i=1}^N$, where $a_i \sim \mu(\cdot), r_i \sim R(a_i)$. The goal of offline learning in MABs can be described as the following constrained optimization problem, where $d$ represents occupancy over actions (arms)
\begin{align} \label{eq:MAB-constrained-optimization-with-d}
    \max_{d \geq 0} \E_{a \sim d} \left[ r(a) \right]
    \qquad \text{s.t.}  \quad \sum_{a} d(a) = 1.
\end{align}
This is the special case of the well-known linear program for reinforcement learning \cite{puterman2014markov}.
\subsection{Primal-dual regularized offline bandits}
To solve \eqref{eq:MAB-constrained-optimization-with-d}, the MIS approach with behavior regularization defines importance weights according to $w(a) = d(a)/\mu(a)$ and converts the problem \eqref{eq:MAB-constrained-optimization-with-d} to its dual form by introducing the Lagrange multiplier $v$:
\begin{align}\label{eq:population-objective-MAB}
    \max_{w \geq 0} \min_{v}  L^{\text{MAB}}_\alpha(w, v) \coloneqq \E_{a \sim \mu} \left[ w(a)r(a) \right] - v\left( \E_{a \sim \mu}[w(a)] - 1 \right) - \alpha \E_{a \sim \mu} \left[ f\left(w(a)\right) \right].
\end{align}
The last term in \eqref{eq:population-objective-MAB} is the behavior regularizer that characterizes the $f$-divergence between the learned occupancy $d$ and data distribution $\mu$, with $\alpha$ determining the strength of regularization. This term was originally proposed to induce \textit{conservatism} by keeping the learned policy close to behavior policy \cite{nachum2019algaedice, lee2021optidice}. The optimization problem \eqref{eq:population-objective-MAB} satisfies strong duality and we denote optimal solutions to the primal and dual variables by $w^\star_\alpha$ $v^\star_\alpha$, whose characteristics are established in Appendix \ref{app:regularized-primal-dual}. Importantly, when $\alpha = 0$, weights $w^\star \coloneqq w^\star_0$ induce an optimal policy and $v^\star \coloneqq v^\star_0$ becomes the optimal value (reward).

Approximating $w$ and $v$ to belong in classes $\cW \subseteq \mathbb{R^+}^{|\cA|}$ and $\cV \subseteq \mathbb{R}$ and solving an empirical approximation of \eqref{eq:population-objective-MAB} yields Algorithm \ref{alg:PRO-MAB}, which we call primal-dual regularized offline MAB (PRO-MAB) as it is a special case of PRO-RL algorithm of \citet{zhan2022offline}.

\begin{algorithm}[h]
\caption{Primal-dual regularized offline multi-armed bandits (PRO-MAB)}
\label{alg:PRO-MAB}
\begin{algorithmic}[1]
\State \textbf{Inputs:} Dataset $\cD = \{(a_i, r_i)\}_{i=1}^N$, classes $\cV$ and $\cW$, function $f(\cdot)$, parameter $\alpha$.
\State Find a solution $\hat w, \hat v$ to the following problem
\begin{align}\label{eq:MAB-empirical-objective}
    \max_{w \in \cW} \min_{v \in \cV} \hat L^{\text{MAB}}_\alpha(w, v) \coloneqq \frac{1}{N}\sum_{i=1}^N w(a_i) r_i  - v (w(a_i) - 1) - \alpha f(w(a_i)). 
\end{align}
\State \textbf{Return:} $\hat \pi = \pi_{\hat w}$
\end{algorithmic}
\end{algorithm}

One might wonder whether the unregularized algorithm ($\alpha = 0$) is sufficient for solving the offline learning problem in MABs, particularly under the natural and common assumption that elements of the function class $\cW$ are bounded, i.e., $w(a) = d(a)/\mu(a) \leq B_w$ for any $w \in \cW$. In the following proposition, we show that the answer is negative and there exist an offline MAB instance in which the unregularized MIS finds a policy that suffers from a constant suboptimality. The proof is provided in Appendix \ref{app:MAB-alpha-0-fails}.

\begin{prop}[Unregularized MIS fails in MABs]\label{prop:MAB-alpha-0-fails} Assume $0 \leq w(a) \leq B_w$ for any $w \in \cW$ and $|v| \leq B_v$ for any $v \in \cV$. Further suppose realizability of $w^\star \in \cW$ and $v^\star \in \cV$ and concentrability of optimal policy $\pi^\star \coloneqq \pi_{w^\star}$. For any $N \geq 2$, there exists a two-armed offline bandit instance where policy $\hat \pi$ returned by Algorithm \ref{alg:PRO-MAB} with $\alpha = 0$ satisfies $J(\pi^\star) - J(\hat \pi) = {1}/{6}$ with a constant probability.
\end{prop}

We note that \citet{zhan2022offline} also argues the failure of the unregularized algorithm by giving a counterexample in the MDP setting. We discuss this example in detail in Section \ref{sec:example}. Proposition \ref{prop:MAB-alpha-0-fails} reveals additional insights to the MDP failure example: the objective \eqref{eq:MAB-empirical-objective} with $\alpha = 0$ fails not just in MDPs but also in bandits, even when the optimal policy is unique and data are collected by running a behavior policy.

Given the failure of the unregularized MIS algorithm, we conduct a tight analysis of PRO-MAB with $\alpha > 0$. In the next theorem, we prove that under similar assumptions as \citet{zhan2022offline} and with a proper choice of hyperparameter 
 $\alpha$, PRO-MAB returns a policy that enjoys near-optimal sample complexity.
\begin{theorem}[Suboptimality of Algorithm \ref{alg:PRO-MAB}]\label{thm:suboptimality-MAB}
Let $f: \mathbb{R} \mapsto \mathbb{R} $ be $M_f$-strongly convex, differentiable, and non-negative with bounded values $|f(x)| \leq B_f$ and bounded first-order derivative $|f'(x)| \leq B_{f'}$. Assume $0 \leq w(a) \leq B_w$ for any $w \in \cW$ and $|v| \leq B_v$ for any $v \in \cV$. Fix $\delta \geq 0$ and set 
\begin{align}\label{eq:alpha-def}
    \alpha = \frac{16((B_w+1)(B_v+1) + B_f)}{M_f} \sqrt{\frac{\log (|\cV||\cW|/\delta)}{N}}.
\end{align}
Suppose realizability of $w^\star_\alpha \in \cW$ and $v^\star_\alpha \in \cV$ and concentrability of $\pi^\star \coloneqq \pi_{w^\star}$ and $\pi^\star_\alpha \coloneqq \pi_{w^\star_\alpha}$ for the $\alpha$ given in \eqref{eq:alpha-def}. Then, with probability at least $1-\delta$, policy $\hat \pi$ returned by Algorithm \ref{alg:PRO-MAB} achieves
\begin{align*}
    J(\pi^\star) - J(\hat \pi) 
    \lesssim \frac{(B_f + B_w(B_v+1))(B_f + B_{f'} B_w)}{M_f} \sqrt{\frac{\log (|\cV||\cW|/\delta)}{N}}.
\end{align*}
\end{theorem} 
To our knowledge, this is the first statistically optimal guarantee for a practical offline MAB algorithm with function approximation and partial coverage and improves over the $1/N^{1/6}$ guarantee given by \citet{zhan2022offline}. We now briefly explain the differences between the analysis methods; a complete proof is deferred to Appendix \ref{app:suboptimality-PRO-MAB}. \citet{zhan2022offline} bounds policy suboptimality by $\alpha + 1/(\alpha^{1/2} N^{1/4})$. Here, the first term $\alpha$ in the bound stems from the bias caused by the regularizer. The second term emerges by connecting the difference between $\hat w$ and $w^\star_\alpha$ to the statistical approximation error via strong convexity of $L^{\text{MAB}}_\alpha$ induced by behavior regularization, which is inversely related to $\alpha$. Optimizing the bound over $\alpha$ gives the final $1/N^{1/6}$ guarantee. In contrast, our analysis connects suboptimality to \emph{occupancy validity}. In particular, we prove that suboptimality is bounded by $(\alpha + 1/ \sqrt{N})/d_{\hat w}$, where $d_{\hat w} = \sum_a \hat w(a) \mu(a)$. We then show that setting $\alpha \asymp 1/\sqrt{N}$ is sufficient to ensure near-validity of occupancy $d_{\hat w} = \Omega(1)$, which yields the statistically optimal rate.

We observe a similar phenomenon in Proposition \ref{prop:MAB-alpha-0-fails} that small $d_w$ for certain $w \in \cW$ can cause the unregularized MIS algorithm to fail. In the following section, we investigate this phenomenon further, leading to a new offline learning algorithm.

\subsection{Augmented Lagrangian replaces behavior regularization}
The next proposition further affirms the importance of occupancy validity. This result shows that if one ensures that the optimization constraint is satisfied, such as by searching only over the weights that induce valid occupancies, then the unregularized algorithm achieves an optimal rate. Proof of this result can be found in Appendix \ref{app:MAB-constraint-sufficient}. 
 
\begin{prop}[Constraint satisfaction is sufficient for unregularized MIS in MAB]\label{prop:MAB-constraint-sufficient} Assume as in Theorem \ref{thm:suboptimality-MAB}. Let $\hat \pi$ be the output of Algorithm \ref{alg:PRO-MAB} with $\alpha = 0$ and assume that $\sum_a \mu(a) \hat w(a) = 1$. Then, for any fixed $\delta > 0$, policy $\hat \pi$ achieves the following bound with probability of as least $1-\delta$
\begin{align*}
    J(\pi^\star) - J(\hat \pi) \lesssim B_w(B_v + 1) \sqrt{\frac{\log (|\cV||\cW|/\delta)}{N}}.
\end{align*} 
\end{prop}

Motivated by the discussion above, we take a step back and ask: are there any other methods for solving constrained optimization problems that find \textit{more constraint-satisfying} solutions when applied to the empirical approximation of the original problem? A promising candidate is the augmented Lagrangian method (ALM) which adds a quadratic loss on the constraints to the objective. Applied to the offline bandits problem \eqref{eq:MAB-constrained-optimization-with-d}, ALM forms the following objective, whose empirical version leads to Algorithm \ref{alg:PAL-MAB}.
\begin{align}\label{eq:AL-population-objective-MAB}
    \max_{w \geq 0} \min_{v}  L^{\text{MAB}}_{\text{AL}}(w, v) \coloneqq \E_{a \sim \mu} \left[ w(a)r(a) \right] - v\left( \E_{a \sim \mu}[w(a)] - 1 \right) - \left( \E_{a \sim \mu}[w(a)] - 1 \right)^2.
\end{align}
Since the last term in \eqref{eq:AL-population-objective-MAB} is zero for the optimal solution $w^\star$, the saddle point solution to \eqref{eq:AL-population-objective-MAB} is equal to the solution to the original constrained optimization problem; see Lemma \ref{lemma:optimal-solution-with-AL-unchanged} for a general result.
The following theorem establishes an upper bound on the suboptimality of the policy returned by Algorithm \ref{alg:PAL-MAB}. This theorem is a special case of Theorem \ref{thm:PAL-CB-suboptimality}, whose proof is given in Appendix \ref{app:PAL-CB-suboptimality}.
 \begin{algorithm}[t]
\caption{Conservative Offline MAB with Augmented Lagrangian}
\label{alg:PAL-MAB}
\begin{algorithmic}[1]
\State \textbf{Inputs:} Dataset $\cD = \{(a_i, r_i)\}_{i=1}^N$, classes $\cW$ and $\cV$.
\State Find a solution $\hat w, \hat v$ to the following problem \vspace{-0.3cm}
\begin{align}\label{eq:MAB-empirical-objective-AL} 
    \max_{w \in \cW} \min_{v \in \cV} \hat L^{\text{MAB}}_{AL}(w, v)  \coloneqq \frac{1}{N}\sum_{i=1}^N w(a_i) r_i - v (w(a_i) - 1) - \Big(\frac{1}{N} \sum_{i=1}^N w(a_i) - 1\Big)^2.
\end{align}      \vspace{-0.3cm}
\State \textbf{Return:} $\hat \pi = \pi_{\hat w}$.
\end{algorithmic}
\end{algorithm}
\begin{theorem}[Suboptimality of Algorithm \ref{alg:PAL-MAB}]\label{thm:suboptimality-PAL-MAB} Assume that $0 \leq w(a) \leq B_w$ for any $w \in \cW$ and $|v| \leq B_v$ for any $v \in \cV$. Further suppose realizability of $w^\star \in \cW$ and $v^\star \in \cV$ and concentrability of $\pi^\star = \pi_{w^\star}$. For any fixed $\delta > 0$, policy $\hat \pi$ returned by Algorithm \ref{alg:PAL-MAB} achieves the following bound with probability of at least $1-\delta$ \vspace{-0.5cm}
\begin{align}\label{eq:ALM-MAB-suboptimality}
    J(\pi^\star) - J(\hat \pi) & \lesssim B_w^2(B_v +1) \sqrt{\frac{\log (|\cW||\cV|/\delta)}{N}}.
\end{align}
\end{theorem} 
With a tabular deterministic parameterization of $\cW$, choosing $\cV = [0,1], \delta = 1/N$, and setting $B_w = C^\star$ to the smallest possible value, the bound in \eqref{eq:ALM-MAB-suboptimality} becomes $ J(\pi^\star) - J(\hat \pi) \lesssim (C^\star)^2 \sqrt{\log (N|\cA|)/N}$. This bound is similar to the suboptimality guarantee $ J(\pi^\star) - J(\hat \pi) \lesssim \sqrt{C^\star \log (N|\cA|)/N}$ of the lower confidence bound (LCB) algorithm for MABs \cite{rashidinejad2021bridging}, except for dependency on $C^\star$. 

In the proof of Theorem \ref{thm:suboptimality-PAL-MAB}, we show that ALM results in near-validity of $\hat w$ by ensuring that $d_{\hat w} = \Omega(1)$, leading to the optimal suboptimality rate. Importantly, Algorithm \ref{alg:PAL-MAB} does not include any explicit form of conservatism through regularizers or uncertainty quantifiers. Colloquially, the MIS formulation and boundedness of $\cW$ elements ensure that $d_{\hat w}(a)/\mu(a) = \hat w(a) \leq B_w$. The ALM term ensures that $d_{\hat w}$ is lower bounded by a constant $c$, which means that the actual occupancy over the arms $\hat \pi(a) = d_{\hat w}(a)/d_{\hat w} \leq  d_{\hat w}(a)/c$. Thus, Algorithm \ref{alg:PAL-MAB} finds a policy $\hat \pi$ that satisfies 
\begin{align*}
    \frac{\hat \pi(a)}{\mu(a)} = \frac{d_{\hat w}(a)}{d_{\hat w}} \frac{1}{\mu(a)} \leq \frac{1}{c} \frac{d_{\hat w}(a)}{\mu(a)} = 
 \frac{1}{c} \hat{w}(a)\leq \frac{B_w}{c},
\end{align*}
i.e., it it finds a policy whose \textit{actual occupancy} is supported by the behavior data distribution $\mu$.

Algorithm \ref{alg:PAL-MAB} offers several benefits compared to PRO-MAB: it only requires $\pi^\star$-concentrability instead of the $\pi^\star, \pi^\star_\alpha$-concentrability requirement of PRO-MAB, removes the need to design regularization function $f$ and adjust $\alpha$, and does not introduce bias in the objective. The main advantage of ALM, however, becomes more evident as we move beyond bandits, where the behavior regularization provably fails to achieve the optimal statistical rate while ALM maintains optimality.

\section{Contextual bandits}
The problem of offline contextual bandits (CB) is a special case of offline RL with $\gamma = 0$ and offline dataset $ \cD = \{(s_i, a_i, r_i)\}_{i=1}^N$, where $s_i \sim \mu(\cdot) = \rho(\cdot)$, $a_i \sim \mu(\cdot| s_i)$, and $r_i \sim R(s_i, a_i)$. The linear programming constrained optimization problem for CB is given by
\begin{align} \label{eq:CB-constrained-optimization-with-d}
    \max_{d \geq 0} \E_{s, a \sim d } \left[ r(s, a) \right]
    \qquad \text{s.t.}  \quad \sum_{a} d(s, a) = \rho(s) \quad \forall s \in \cS.
\end{align}
\subsection{Analysis of the primal-dual regularized offline contextual bandits}
Similar to the MAB setting, the offline learning problem \eqref{eq:CB-constrained-optimization-with-d} can be turned into the primal-dual form with behavior regularization, leading to the primal-dual regularized offline CB (PRO-CB). Details of the PRO-CB objective, optimal primal and dual variables $w^\star_\alpha, v^\star_\alpha$, and pseudocode are provided in Appendix \ref{app:PRO-CB}. In the following proposition, we prove a performance lower bound on the PRO-CB algorithm, whose proof is presented in Appendix \ref{app:PRO-CB-suboptimal}.

\begin{prop}[Performance lower-bound on Algorithm \ref{alg:PRO-CB} (PRO-CB)]\label{prop:PRO-CB-suboptimal}
Let $f: \mathbb{R} \mapsto \mathbb{R} $ be $M_f$-strongly convex, differentiable, and non-negative with bounded values $|f(x)| \leq B_f$ and bounded first-order derivative $|f'(x)| \leq B_{f'}$. Assume $0 \leq w(s, a)\leq B_w$ for $w \in \cW$ and $|v(s)| \leq B_v$ for $v \in \cV$. Suppose realizability of $w^\star, w^\star_\alpha \in \cW$ and $v^\star, v^\star_\alpha \in \cV$ and concentrability of $\pi^\star, \pi^\star_\alpha$. Let $\hat \pi$ be the output of Algorithm \ref{alg:PRO-CB}. Further, assume that $N \geq n_0$, where $n_0$ is a polynomial function of $\delta, B_w, B_v, B_f, B_{f'}$. Then, for any $\alpha \geq 0$ there exists an offline CB instance such that $J(\pi^\star) - J(\hat \pi)  \gtrsim N^\beta$ with a constant probability, where $\beta > -1/2$.
\end{prop} 
Proposition \ref{prop:PRO-CB-suboptimal} shows that behavior regularization fails to achieve the optimal $1/\sqrt{N}$ rate regardless of $\alpha$, even under boundedness, realizability, and concentrability assumptions. The main takeaway of our construction in the proof of Proposition \ref{prop:PRO-CB-suboptimal} is that ensuring occupancy validity $\sum_a \hat w(s,a) \mu(a|s) = \Omega(1)$ for nearly all states appears to be critical in achieving the optimal rate. Yet, without introducing a significant bias, behavior regularization is insufficient to induce the state-wise occupancy validity.

\subsection{Conservative offline CB with augmented Lagrangian}
To enforce the occupancy validity constraints, we propose to incorporate augmented Lagrangian in the following form:
\small 
\begin{align}\label{eq:AL-population-objective-CB}
    \max_{w \geq 0} \min_{v}  L^{\text{CB}}_{\text{AL}}(w, v)
    \coloneqq \E_{\mu} \left[ w(s,a)r(s,a) \right] 
    - \E_{\mu}[v(s) (w(s,a) - 1)] - \E_{s \sim \mu} [(\E_{a \sim \mu(\cdot|s)}[w(s,a)] - 1)^2]
\end{align}\normalsize
Notice that when $|\cS| = 1$, \eqref{eq:AL-population-objective-CB} simplifies to the ALM objective \eqref{eq:population-objective-MAB} in the MAB setting. The ALM term can be understood as follows: each element in the ALM sum encourages the validity of occupancy in each state $\sum_a w(s,a) \mu(a|s) \approx 1$, and the elements are weighted according to the true state distribution since validity is more important in states that are actually more likely to be visited. As before, denote by $w^\star$ the optimal solution to \eqref{eq:AL-population-objective-CB}, which is equal to the optimal solution of the original constrained optimization problem. Additionally, we define $v^\star(s) \coloneqq V^\star(s)$ to be equal to the optimal reward at each state $s \in \cS$.

\begin{algorithm}[h]
\caption{Conservative Offline CB with Augmented Lagrangian}
\label{alg:PAL-CB}
\begin{algorithmic}[1]
\State \textbf{Inputs:} Dataset $\cD = \{(s_i, a_i, r_i)\}_{i=1}^N$, function classes $\cW, \cV$
\State Find a solution $\hat w, \hat v$ to the following problem
\begin{align}\label{eq:CB-empirical-objective-AL}
        \max_{w \in \cW} \min_{v \in \cV} \hat L^{\text{CB}}_{\text{AL}}(w, v) \coloneqq \frac{1}{N}\sum_{i=1}^N w(s_i, a_i) (r_i - v(s_i)) + v(s_i) - \Big(\sum_{a \in \cA} w(s_i, a) \mu(a|s_i) - 1\Big)^2
\end{align}   
\State \textbf{Return:} $\hat \pi = \pi_{\hat w}$.
\end{algorithmic}
\end{algorithm}

Solving a sample-based approximation to objective \eqref{eq:AL-population-objective-CB} and using function approximation for $w$ and $v$ gives us conservative offline CB with augmented Lagrangian as given in Algorithm \ref{alg:PAL-CB}. We analyze the suboptimality of Algorithm \ref{alg:PAL-CB} and present the following theorem, showing that the ALM achieves the optimal rate without requiring behavior regularization. The proof of this theorem can be found in Appendix \ref{app:PAL-CB-suboptimality}.
\begin{theorem}[Suboptimality of Algorithm \ref{alg:PAL-CB}]\label{thm:PAL-CB-suboptimality} 
Assume $0 \leq w(s, a)\leq B_w$ for $w \in \cW$ and $v(s) \leq B_v$ for $v \in \cV$.  Moreover, suppose realizability of $w^\star \in \cW$ and $v^\star \in \cV$ and concentrability of $\pi^\star = \pi_{w^\star}$. For any fixed $\delta \geq 0$, policy $\hat \pi$ returned by Algorithm \ref{alg:PAL-CB} achieves the following suboptimality bound with probability of at least $1-\delta$
\begin{align*}
    J(\pi^\star) - J(\hat \pi) \lesssim B_w^2(B_v +1) \sqrt{\frac{\log (|\cW||\cV|/\delta)}{N}}.  
\end{align*}
\end{theorem}
 
\section{Markov decision processes}
We now turn our focus to offline RL. In addition to the offline dataset of interactions, we assume access to a dataset $\cD_0 = \{s_i\}_{i=1}^{N_0}$ with i.i.d.~samples from the initial distribution $\rho$, similar to prior works \cite{lee2021optidice, zhan2022offline}. The linear programming formulation of RL \cite{puterman2014markov} solves the following constrained optimization problem:
\begin{align}\label{eq:MDP-LP-program}
    \max_{d \geq 0} \E_{s,a \sim d} [r(s,a)] \qquad 
    \text{s.t.} \quad & d(s) = (1-\gamma) \rho(s) + \gamma \sum_{s', a'} P(s|s', a') d(s', a') \quad \forall s \in \mathcal{S}. 
\end{align}
The constraints are known as the Bellman flow equations and restrict the search to the space of valid occupancy distributions $d^\pi$ that can be induced in the MDP by running a policy $\pi$.
   
\subsection{Conservative offline RL with augmented Lagrangian}
Motivated by the success of ALM in bandits, we propose the following extension to offline RL:   
\begin{align}\label{eq:MDP-population-AL-objective}
        \max_{w \geq 0} \min_{v} L^{\text{MDP}}_{\text{AL}} (w, v) \coloneqq (1-\gamma) \E_{\rho} [v(s)] + \E_{\mu} \left[w(s,a) e_v(s,a)\right] - \E_{d^{\pi_w}} \left[ \left(\frac{d_w(s)}{d^{\pi_w}(s)} - 1\right)^2 \right],
\end{align}
where $e_v(s,a) \coloneqq r(s,a) + \gamma \sum_{s'} P(s'|s,a) v(s') - v(s)$. 
One can check that the first two terms equate to the Lagrange dual of \eqref{eq:MDP-LP-program} and the last term is a generalization of the ALM terms in bandits. Each iterand of the ALM sum encourages the occupancy $d_{w}(s)$ to be close in ratio to the actual occupancy $d^{\pi_{w}}(s)$ in each state and as before, the ALM iterands are weighted according to actual state visitations $d^{\pi_{w}}(s)$. Our particular ALM construction can be intuitively understood as follows: the MIS formulation learns bounded weights $\hat w(s,a) = d_{\hat w}(s,a)/\mu(s,a) \leq B_w$, and the ALM term ensures that $d_{\hat w}(s,a)/d^{\pi_{\hat w}}(s,a) = d_{\hat w}(s)/d^{\pi_{\hat w}}(s) = \Omega(1)$ for most states, which translates to ${d^{\pi_{\hat w}}(s,a)}/{\mu(s,a)} \lesssim B_w$.

The ALM term in \eqref{eq:MDP-population-AL-objective} is difficult {to estimate} as it involves an expectation over the unknown occupancy $d^{\pi_w}$ and computing the ratio $d_w(s)/d^{\pi_w}(s)$. In the following section, we resolve this difficulty by converting the ALM term into a variational form.
\subsection{Estimating the ALM term and CORAL algorithm}
We view the ALM term as the negative $f$-divergence\footnote{Although $d_w$ may not be a valid distribution, the variational form still holds. The case of $f(x) = (x-1)^2$ corresponds to the chi-squared divergence but many of our results hold more generally.} between $d_w$ and $d^{\pi_w}$ with $f(x) \coloneqq (x-1)^2$ and express it in the variational form \cite{nguyen2010estimating}:
\begin{align} \label{eq:f-divergence-variational-form}
    - \E_{d^{\pi_w}} \left[ \left(\frac{d_w(s)}{d^{\pi_w}(s)} - 1\right)^2 \right]
    = - D_f(d_w \| d^{\pi_w}) = \min_{x} \E_{d^{\pi_w}} [f_*(x(s,a))] - \E_{d_w} [x(s,a)].
\end{align}
Here, $f_*$ is the convex conjugate of $f$ and we used the fact that $d_w(s,a)/d^{\pi_w}(s,a) = d_w(s)/d^{\pi_w}(s)$. Notice that $\E_{d^{\pi_w}} [f_*(x(s,a))]$ is the value of $\pi_w$ in the same MDP but with rewards $f_*(x(s,a))$. Define $u$ as the fixed point of the following Bellman equation  
\begin{align}
\label{eq:Bellman_u_and_x}
    u(s,a) \coloneqq f_*(x(s,a)) + \gamma (\mathbb{P}^{\pi_w} u)(s,a).
\end{align}
Since $u(s,a)$ is the state-action value function (Q-function) of $\pi_w$ with rewards $f_*(x(s,a))$, we can rewrite \eqref{eq:f-divergence-variational-form} as   
\begin{align}\label{eq:AL-variational-form-model-based}
    \eqref{eq:f-divergence-variational-form} = \min_u (1-\gamma) \E_{s \sim \rho, a \sim \pi_w} [u(s, a)] - \E_\mu \left[ w(s,a) f_*^{-1}\left(u(s,a) - \gamma  (\mathbb{P}^{\pi_w} u)(s,a)\right)\right].
\end{align}
Equation \eqref{eq:AL-variational-form-model-based} involves expectations over $\rho$ and $\mu$, which can be estimated empirically using interaction dataset $\cD$ and dataset of initial states $\cD_0$, yet, it also includes a term that involves the transition operator $\mathbb{P}^{\pi_w}$. In the rest of this section, we discuss model-free and model-based methods for estimating the term involving the transition operator $\mathbb{P}^{\pi_w}$. We include some details on practical implementations in Appendix \ref{app:practical-implementation}.

\begin{algorithm}[b]
\caption{Conservative Offline RL with Augmented Lagrangian (CORAL) --- Model-based}
\label{alg:PAL-MDP-model-based}
\begin{algorithmic}[1]
\State \textbf{Inputs:} Datasets $\cD$, $\cD_0$, $\cD_m$, function classes $\cW, \cV, \cU, \cP$, $f_*^{-1}(x) = 2 \sqrt{x + 1} -2$.
\State Estimate transitions via maximum likelihood:
$ \hat P = \argmax_{P \in \cP} \sum_{i=1}^{N_m} \ln P(s'_i|s_i, a_i).$
\State Find a solution $\hat w, \hat v, \hat u$ to the following problem 
\begin{align}\label{eq:MDP-empirical-objective-AL-model-based}
    \begin{split}
        & \max_{w \in \cW} \min_{v \in \cV} \min_{u \in \cU} \hat L^{\text{model-based}}_{AL}(w, v)  \coloneqq \frac{(1-\gamma)}{N_0}\sum_{i=1}^{N_0} \Big(v(s_i)  +  \sum_a u(s_i,a) \pi_w(a|s_i)\Big)\\ 
    & \;+ \frac{1}{N} \sum_{i=1}^N w(s_i, a_i) \left[r_i + \gamma v(s_i') - v(s_i) - f_*^{-1}\left(u(s_i, a_i) - \gamma (\hat{\mathbb{P}}^{\pi_w} u)(s_i,a_i)\right)  \right] 
    \end{split}
\end{align}     
\State \textbf{Return:} $\hat \pi = \pi_{\hat w}$.
\end{algorithmic}
\end{algorithm}

\subsubsection{Model-based CORAL}
The model-based ALM population objective is obtained by directly substituting the variational form \eqref{eq:AL-variational-form-model-based} in the original objective \eqref{eq:MDP-population-AL-objective} 
\begin{align}\label{eq:population-objective-model-based-MDP}
    \begin{split}
    & \max_{w \geq 0} \min_{v} \min_{u} L_{\text{AL}}^{\text{model-based}}(w, v, u) \coloneqq (1-\gamma) \E_{s \sim \rho} \left[v(s) + \sum_a u(s,a) \pi_w(a|s) \right] \\
    & \quad  + \E_{s,a \sim \mu} \left[w(s,a) \left(e_v(s,a) - f_*^{-1} \left(u(s,a) - \gamma (\mathbb{P}^{\pi_w} u)(s,a)\right) \right)\right]
    \end{split}
\end{align}
Note that optimal solution $w^\star$ of \eqref{eq:MDP-population-AL-objective} is also the optimal solution to \eqref{eq:population-objective-model-based-MDP} since $d_{w^\star} = d^{\pi_{w^\star}}$ (c.f. Lemma \ref{lemma:optimal-solution-with-AL-unchanged}).

To obtain a sample-based approximation of the above objective, we assume access to a realizable function class $\cP$ that contains the true transitions and an additional dataset on interactions $\cD_m = \{ (s_i, a_i, s_i')\}_{i=1}^{N_m}$, where $s_i, a_i \sim \mu$ and $s_i' \sim P(\cdot \mid s_i, a_i)$. Given $\cD_m$, we obtain a maximum likelihood estimate of transitions and then approximate the expectations using $\cD_0$ and $\cD$. We assume access to the independent dataset $\cD_m$ to simplify the proofs and practical implementations can reuse interaction dataset $\cD$ instead. This leads to Algorithm \ref{alg:PAL-MDP-model-based}, which we call model-based conservative offline RL with augmented Lagrangian (CORAL).

\subsubsection{Model-free CORAL}

As an alternative, we consider developing a model-free that uses a single-sample estimate of $f_*^{-1}\big(u(s,a) - \gamma (P^{\pi_w} u)(s,a)\big)$. Using a single-sample estimate leads to the infamous double sampling problem \citep{baird1995residual}. To circumvent this difficulty, in Appendix \ref{app:model-free-objective} we use the dual embedding trick in \citet{nachum2019dualdice} to derive the following model-free population objective
\begin{align}\label{eq:population-objective-model-free-MDP}
    \begin{split}
       & \max_{w \geq 0} \min_{v} \min_{u} \max_{\zeta < 0} L_{\text{AL}}^{\text{model-free}}(w,v,u,\zeta)= (1-\gamma) \E_{s \sim \rho} \left[v(s) + \sum_a u(s,a) \pi_w(a|s) \right]   \\ & \ \ +
    \E_{(s,a,s') \sim \mu, a' \sim \pi_w(\cdot|s')} [ w(s,a) \left( e_v(s,a) + (u(s,a) - \gamma u(s',a'))\zeta(s,a) +g_*(\zeta(s,a) ) \right) ],
    \end{split}
\end{align}
where $g_*(x) \coloneqq x+2+1/x $. Empirical approximation of objective \eqref{eq:population-objective-model-free-MDP} leads to model-free CORAL presented in Algorithm \ref{alg:PAL-MDP-model-free}.

\begin{algorithm}[h]
\caption{Conservative Offline RL with Augmented Lagrangian (CORAL) --- Model-free}
\label{alg:PAL-MDP-model-free}
\begin{algorithmic}[1]
\State \textbf{Inputs:} Datasets $\cD$, $\cD_0$, function classes $\cW, \cV, \cU, \cZ$, $g_*(x) = x + 2 + \frac{1}{x}$.
\State Find a solution $\hat w, \hat v, \hat u, \hat \zeta$ to $\max_{w \in \cW} \min_{v \in \cV} \min_{u \in \cU} \max_{\zeta \in \cZ} \hat L_{\text{AL}}^{\text{model-free}}(w,v,u,\zeta)$ defined as    
\begin{align}\label{eq:MDP-empirical-objective-AL-model-free}
    \begin{split}
        & \frac{(1-\gamma)}{N_0}\sum_{i=1}^{N_0} v(s_i)  +  \sum_a u(s_i,a) \pi_w(a|s_i) + \frac{1}{N} \sum_{i=1}^N w(s_i, a_i) \Big[r_i + \gamma v(s_i') - v(s_i) \\
        & \quad +\zeta(s_i,a_i)\Big(u(s_i,a_i) - \gamma \sum_{a' \in \cA} u(s'_i, a') \pi_w(a'|s'_i )\Big) + g_*(\zeta(s_i,a_i) )  \Big] 
    \end{split}
\end{align}\normalsize
\State \textbf{Return:} $\hat \pi = \pi_{\hat w}$.
\end{algorithmic}
\end{algorithm}
\subsection{CORAL performance upper bound}
Theorem \ref{thm:CORAL-suboptimality} shows that both variants of CORAL enjoy optimal rates; see Appendix \ref{app:CORAL-suboptimality} for the proof.
\begin{theorem}[CORAL Suboptimality]\label{thm:CORAL-suboptimality} 
Assume $0 \leq w(s,a) \leq B_w$ for $w \in \cW$, $|v(s)| \leq B_v$ for $v \in \cV$, and $|u(s,a)| \leq B_u$. Suppose realizability of $w^\star \in \cW$ and $v^\star(s) = V^\star(s) \in \cV$ and concentrability of $\pi^\star = \pi_{w^\star}$. Let $\tilde x_w (s,a) = \clip (x^\star_w(s,a), -B_x, B_x)$, where $x^\star_w$ is a solution to \eqref{eq:f-divergence-variational-form} and $B_x = (1-\gamma)/4$, and define $u^\star_w$ as the fixed-point solution to \eqref{eq:Bellman_u_and_x} when $x = \tilde x_w$. Assume $u^\star_w \in \cU$ for any $w \in \cW$. Then, $B_u$ satisfies $(1-\gamma)^{-1}(B_x^2/4 + B_x) \leq B_u \leq \frac{1}{2}$. Moreover, for any fixed $\delta \geq 0$, the following statements hold:
\begin{enumerate}[label= (\Roman*)]
    \item Assume $N = N_0 = N_m$ for simplicity. If $P^\star \in \cP$, then $\hat \pi$ returned by Algorithm \ref{alg:PAL-MDP-model-based} achieves 
    \begin{align*}
        J(\pi^\star) - J(\hat \pi) \lesssim \frac{B_v + B_u + (1+B_v)B_w}{(1-\gamma)^3}\sqrt{\frac{B_u\log(|\cP||\cU||\cW||\cV|/\delta) }{N}}.
    \end{align*}
    \item Assume $N = N_0$ for simplicity. Let $\zeta^\star_{w, u} = \argmax_{\zeta < 0} L^{\text{model-free}}_{AL}(w, v, u, \zeta)$ defined in \eqref{eq:population-objective-model-free-MDP}. Assume $\zeta^\star_{w^\star, u} \in \cZ$ for $u \in \cU$ and $B_{\zeta, L} \leq |\zeta(s,a)| \leq B_{\zeta, U}$ for $\zeta \in \cZ$, where $B_{\zeta, L} \in (0, 2/(2+B_x))$ and $B_{\zeta, U} \geq 2/(2-B_x)$.  Let $B_{\zeta} = \max \{ B_{\zeta, U}, B_{\zeta,L}^{-1}\}$. Then, $\hat \pi$ returned by Algorithm \ref{alg:PAL-MDP-model-free} achieves   
    \begin{align*}
        J(\pi^\star) - J(\hat \pi) \lesssim \frac{B_v + B_u + (1+B_v +B_{\zeta}(B_u+1))B_w}{(1-\gamma)^3}\sqrt{\frac{\log(|\cU||\cW||\cV||\cZ|/\delta) }{N}}.
    \end{align*}
\end{enumerate}
\end{theorem}
In Theorem \ref{thm:CORAL-suboptimality}, we make realizability assumptions on $u^\star_w \in \cU$ for $w \in \cW$ and $\zeta^\star_{w^\star, u} \in \cZ$ for $u \in \cU$. Such assumptions are common in the theory of RL with function approximation \cite{munos2008finite,xie2021bellman, jiang2020minimax} and removing them can be difficult or even impossible. For example, Bellman completeness is proved to be necessary for polynomial sample complexity in value-based methods \cite{foster2021offline}. Recently, \citet{zhan2022offline, chen2022offline} propose algorithms that only require optimal solution realizability, however, these algorithms are either computationally intractable or statistically suboptimal.

\subsection{Example: Behavior regularization vs. augmented Lagrangian} \label{sec:example}

    \begin{figure}[b]
    \centering
    \scalebox{1}{
    \begin{tikzpicture}[observed/.style={circle, draw=black, fill=black!10, thick, minimum size=10mm},
    hidden/.style={circle, draw=black, thick, minimum size=5mm},
    squarednode/.style={rectangle, draw=black, fill=black, very thick, minimum size=3mm},]
    \node[hidden] (A) at (0,1.1) {$A$};
    \node[hidden] (C) at (0.8,0) {$C$};
    \node[hidden] (B) at (-0.8,0) {$B$};
    \node[squarednode] (t1) at (-0.8,-1.2) {$ $};
    \node at (-0.8,-1.9) {{+1}};
    \node[squarednode] (t2) at (0.3,-1.2) {$ $};
    \node at (0.3,-1.7) {{+0}};
    \node at (0.3,-2.1) {{+1}};
    \node[squarednode] (t3) at (1.3,-1.2) {$ $};
    \node at (1.3,-1.7) {{+1}};
    \node at (1.3,-2.1) {{+0}};
    \draw[dashed, line width=1pt, citeColor] (-0.1,-2.4+0.9)--(-0.1,-0.8+1.3)--(1.7,-0.8+1.3)--(1.7,-2.4+0.9)--(-0.1,-2.4+0.9);
    \draw[->, thick, >=stealth] (A.south east) -- (C.north)  node[pos=0.4,sloped,above] {\small R};
    \draw[->, thick, >=stealth] (A.south west) -- (B.north)  node[pos=0.4,sloped,above] {\small L};
    \draw[->, thick, >=stealth] (B.south) -- (t1.north);
    \draw[->, thick, >=stealth] (C.south west) -- (t2.north) node[pos=0.4,sloped,above] {\small L};
    \draw[->, thick, >=stealth] (C.south east) -- (t3.north) node[pos=0.4,sloped,above] {\small R};
    \end{tikzpicture}}  
    \caption{The agent always starts from state $A$. Action $L$ leads to state $B$, from where the agent collects a +1 reward. Action $R$ leads to state $C$, from where only one action leads to a +1 reward. Nature decides which MDP is presented to the learner. Data distribution is $\mu(A, L) = 1/4, \mu(A, R) = 1/2, \mu(B) = 1/4, \mu(C) = 0$, which satisfies $\pi_{w_1}$-concentrability.}
    \label{fig:MDP_example}
\end{figure}
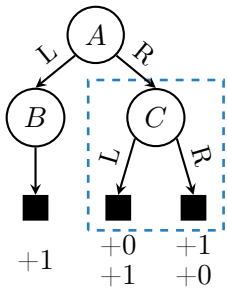

We examine a hard offline RL instance (Figure \ref{fig:MDP_example}) developed by \citet{zhan2022offline} and we compare the performance of unregularized MIS, PRO-RL (MIS with behavior regularization), and MIS with ALM. Assume $\cV = \{v^\star\}$ and $\cW = \{w_1, w_2\}$, where $w_1$ always selects $L$ from $A$ and $w_2$ always selects $R$ from $A$. One can check $w_1(A,L) = 2, w_1(A,R) = 0$ and $w_2(A,L) = 0, w_2(A,R) = 1$.

\paragraph{Unregularized algorithm.} As \citet{zhan2022offline} state, the unregularized algorithm, which solves the objective \eqref{eq:MDP-population-AL-objective} with $\alpha = 0$, fails to distinguish between $w_1$ and $w_2$ even with infinite data since the objectives at $w_1$ and $w_2$ are exactly equal.

\paragraph{Behavior regularization.} Consider an instantiation of PRO-RL with regularizer $-\alpha \E_\mu [w^2(s,a)]$. Since in this example $\E_\mu[w_1^2(s,a)] > \E_\mu[w_2^2(s,a)]$, PRO-RL picks the wrong weight $w_2$ and thus suffers from a constant suboptimality. However, note that this does not contradict theoretical guarantees of PRO-RL as these guarantees additionally assume concentrability of $\pi^\star_\alpha$. Intuitively, behavior regularization causes $\pi^\star_\alpha$ to be more stochastic and thus requiring $\mu(s,a) > 0$ for more states and actions. Here, since $\mu$ covers both $(A,L)$ and $(A,R)$, behavior regularization causes $\pi^\star_\alpha(R|A) > 0$ and thus $d^{\pi^\star_\alpha}(C) > 0$. Therefore to handle the MDP in Figure \ref{fig:MDP_example}, PRO-RL additionally requires $\mu(C) > 0$ to satisfy $\pi^\star_\alpha$-concentrability.

\paragraph{ALM.} In this example, ALM successfully picks the optimal $w_1$, as it avoids a mismatch between the actual and learned occupancies. This is because in \eqref{eq:MDP-population-AL-objective} the ALM term is zero at $w_1$ due to realizability whereas at $w_2$, it has a lower bound $\mathbb{E}_{s \sim d^{\pi_{2}}} \left({d_{w_2}(C)}/{d^{\pi_2}(C)} - 1\right)^2 \geq d^{\pi_{2}}(C) > 0$.

\section{Discussion}
We present a set of practical and statistically optimal algorithms for offline MAB, CB, and RL, under general function approximation and single-policy concentrability. Our algorithms are designed within the MIS formulation combined with a novel application of the augmented Lagrangian method. Importantly, our optimality guarantees hold under MIS combined with ALM alone, without any additional form of conservatism such as via regularization or uncertainty quantification. Furthermore, we investigate the role of regularizers in MIS algorithms. Although the empirical benefits of such regularizers are often attributed to conservatism, our analysis suggests that conservatism stems from the MIS formulation while the role of regularizers is to ensure the validity of learned occupancy. 

Apart from RL, our work on bandits is related to the selection problem \cite{hong2021review}, though the majority of works in this area are in the online setting. Additionally, in our analysis, we solve a subset of stochastic optimization problems with possibly large or infinite stochastic constraints involving conditional expectations. To our knowledge, finite-sample properties of such stochastic optimization problems have not been addressed \cite{shapiro2021lectures} and our work may open up avenues for further research in this area. Other interesting future directions include conducting empirical evaluations of ALM, examining the possibility of removing strong realizability assumptions, and investigating practical and optimal offline RL algorithms whose guarantees hold under milder variants of single-policy concentrability.

\subsubsection*{Acknowledgments}
The authors are grateful to Amy Zhang and Yuandong Tian. 
This work occurred under Meta AI-BAIR Commons at the University of California, Berkeley. PR is supported by the Open Philanthropy Foundation.  Part of the work was done when HZ was a visiting researcher at Meta.

\bibliographystyle{plainnat}
\bibliography{references}

\newpage 

\appendix
\section{Related work}\label{app:related-work}

We covered a number of related works in the introduction and throughout the paper. In this section, we review more related literature.

\subsection{Concentrability assumptions}
 The lack of sufficient coverage in the offline dataset is one of the main challenges in offline RL. In RL theory, dataset coverage has often been characterized by concentrability definitions \cite{munos2007performance, scherrer2014approximate}. Earlier works on offline RL impose all-policy concentrability on the density ratio for all states and actions \cite{scherrer2014approximate,liu2019neural,chen2019information,jiang2019value,wang2019neural,liao2020batch,zhang2020variational}, with some requiring this ratio to be bounded for every time step \cite{szepesvari2005finite, munos2007performance, antos2008learning,farahmand2010error,antos2007fitted}. The works \citet{xie2020batch, feng2019kernel, uehara2020minimax} use slightly milder definitions, such as requiring a bound on a weighted norm of density ratios. The work \citet{xie2020batch} makes even stronger assumptions such as lower bounded conditionals $\mu(a|s)$ and exploratoriness of state marginals to circumvent the Bellman completeness requirement. 

To handle partial coverage, recent algorithms are analyzed based on variants of single-policy concentrability \cite{rashidinejad2021bridging}. Some variants such as the ones presented in works \citet{uehara2021pessimistic} (model-based) or \citet{xie2021bellman,song2022hybrid} (model-free) are more suited to function approximation as they avoid bounded ratio assumption for all states and actions. However, existing offline RL algorithms based on these weaker definitions are either computationally intractable \cite{uehara2021pessimistic, xie2021bellman} or their statistical rate is suboptimal \cite{pmlr-v162-cheng22b}. The most related works are 
\citet{zhan2022offline}, which requires two-policy concentrability, and \citet{chen2022offline}, which requires single-policy concentrability on density ratio for all states and actions.

\subsection{Conservative offline RL}

A series of recent works on offline RL have focused on addressing partial coverage of offline dataset through conservative algorithm design. Broadly speaking, these methods can be broken down into several categories. The first category of methods applies policy constraints, enforcing the learned policy to be close to the behavior policy. Such constraints are applied either explicitly \cite{fujimoto2019off,ghasemipour2020emaq,jaques2019way, siegel2020keep, kumar2019stabilizing,wu2019behavior, fujimoto2021minimalist}, implicitly \cite{peng2019advantage, nair2020accelerating}, or through importance sampling \cite{liu2019off, swaminathan2015batch, nachum2019algaedice, lee2021optidice, zhang2020gradientdice, Zhang2020GenDICE:}. Another category involves learning conservative values such as conservative Q-learning \cite{kumar2020conservative}, fitted Q-iteration with conservative update \cite{liu2020provably}, subtracting penalties \cite{rezaeifar2022offline}, and critic regularization \cite{kostrikov2021offline}. The last category includes model-based methods such as learning pessimistic models \cite{kidambi2020morel, guo2022model}, adversarial model learning \cite{rigter2022rambo}, forming penalties using model ensembles \cite{yu2020mopo}, or incorporating a combination of model and values \cite{yu2021combo}.

On the theoretical side, as discussed in the introduction, the majority of works design pessimistic offline RL algorithms that rely on some form of uncertainty quantification \cite{yin2021towards, uehara2021representation, zhang2022corruption, yan2022efficacy, yin2022near, kumar2021should, shi2022distributionally, wang2022gap}. One exception is the work of \citet{zanette2021provable} that uses value-function perturbation with actor-critic in linear function approximation setting. Other examples include the recent theoretical works on MIS \cite{zhan2022offline, chen2022offline} and adversarially trained actor-critic \cite{pmlr-v162-cheng22b}.

Most related to our work are methods that focus on provable conservative offline RL under general function approximation and partial coverage. \citet{uehara2021pessimistic} propose a pessimistic model-based algorithm that under a generalization of single-policy concentrability to bounded TV distance ratio, enjoys a $1/\sqrt{N}$ rate but is computationally intractable. The work of \citet{xie2021bellman} presents a pessimistic model-free algorithm under a variant of single-policy concentrability framework that requires a bounded ratio of average Bellman error and Bellman completeness. While the original version of the algorithm achieves the optimal $1/\sqrt{N}$ rate, it is computationally intractable. A practical version of the algorithm is presented and has a suboptimal $1/N^{1/5}$ guarantee. Another related work by \citet{chen2022offline} studies MIS combined with value function approximation under $\pi^\star$-concentrability and proves a $1/\sqrt{\text{gap}(Q^\star)N}$ rate, yet the guarantee degrades with $Q^\star$ gap and the algorithm is computationally intractable. \citet{pmlr-v162-cheng22b} propose an adversarially trained actor-critic method that enjoys provable $1/N^{1/3}$ rate under the single-policy concentrability definition of \citet{xie2021bellman} and Bellman completeness and performs well in offline RL benchmarks when combined with deep neural networks.

\section{Proofs for multi-armed bandits}

We start in Appendix \ref{app:regularized-primal-dual} by characterizing the bias caused by adding the behavior regularization in \eqref{eq:population-objective-MAB}. In Appendix \ref{app:MAB-alpha-0-fails}, we prove Proposition \ref{prop:MAB-alpha-0-fails} that demonstrates the failure of unregularized MIS for solving offline MABs, even when the optimal solutions are realizable and the optimal policy is covered in the offline data. Appendix \ref{app:suboptimality-PRO-MAB} is devoted to the proof of Theorem \ref{thm:suboptimality-MAB}, which gives a tight performance upper bound of the PRO-MAB algorithm. Finally in Appendix \ref{app:MAB-constraint-sufficient}, we prove Proposition \ref{prop:MAB-constraint-sufficient}, showing that constraint satisfaction is sufficient for the success of unregularized MIS.

\subsection{Solutions to the primal-dual regularized objective}\label{app:regularized-primal-dual}

In the following lemma, we characterize the optimal solution $(w^\star_\alpha, v^\star_\alpha)$ to the behavior-regularized population objective \eqref{eq:population-objective-MAB} as well as the suboptimality of the policy induced by $w^\star_\alpha$. 

\begin{lemma}[Regularized primal-dual solutions, MAB]\label{lemma:regularized-unregularized-closeness-MAB} Let $f$ be differentiable, strictly convex, nonnegative, and bounded by $B_f$. Denote $r^\star \coloneqq \max_{a \in \cA} r(a)$. Then, the following statements hold:
\begin{enumerate}[label=(\Roman*)]
    \item\label{claim:wstar0-equal-wstar} $w^\star_0 = w^\star$, where $w^\star$ is the importance weight corresponding to an optimal policy;
    \item\label{claim:vstaralpha-bound} $v^\star_\alpha = r^\star - c \alpha$, where $0 \leq c \leq f'(C^\star)$;
    \item\label{claim:suboptimality-pistar-alpha} policy $\pi^\star_\alpha \coloneqq \pi_{w^\star_\alpha}$ satisfies $J(\pi^\star) - J(\pi^\star_\alpha) \leq \alpha B_f$. 
\end{enumerate}
\end{lemma}
\begin{proof}
Part \ref{claim:wstar0-equal-wstar} follows directly by strong duality. For part \ref{claim:vstaralpha-bound}, notice that KKT conditions imply the following relation between $w^\star_\alpha(a)$ and $v^\star_\alpha$:
\begin{align*}
    w_\alpha^\star(a)=\max\left\{0,(f')^{-1}\left(\frac{r(a)-v_\alpha^\star}{\alpha}\right)\right\}.
\end{align*}
Since $f$ is strictly convex, $f'$ is a monotonically increasing function. Therefore, the optimal arm $a^\star$ has the largest $w^\star_\alpha(a)$, which should be nonzero due to realizability of $w^\star_\alpha$. In other words,
\begin{align}\label{eq:vstar-alpha-equation-MAB}
    w_\alpha^\star(a^\star)=(f')^{-1}\left(\frac{r^\star -v_\alpha^\star(s)}{\alpha}\right)\ \Rightarrow
    v_\alpha^\star=r^\star -\alpha f'(w_\alpha^\star(a^\star)).
\end{align}
We now proceed to find a bound on $f'(w_\alpha^\star(a^\star))$. 
Since $w^\star_\alpha$ is the optimal solution to \eqref{eq:population-objective-MAB}, it must satisfy the constraint
\begin{align*}
    \sum_{a \in \cA} \mu(a) w_\alpha^\star(a) = 1 \Rightarrow w_\alpha^\star(a^\star) \leq \frac{1}{\mu(a^\star)} \leq C^\star,
\end{align*}
where the last inequality stems from the single-policy concentrability assumption of $\pi^*$. Since $f'$ is an increasing function, we have $f'(w^\star_\alpha(a^\star)) \leq f'(C^\star)$, which combined with \eqref{eq:vstar-alpha-equation-MAB} yields the following lower bound on $v^\star_\alpha$
\begin{align*}
    v^\star_\alpha \geq r^\star - \alpha f'(C^\star).
\end{align*}
Moreover, the convexity of $f$ immediately gives the upper bound on $v^\star_\alpha \leq r^\star$, which completes the proof of part \ref{claim:vstaralpha-bound}.

We now prove the last part. Since $w_\alpha^\star$ is the optimal solution to the regularized population objective \eqref{eq:population-objective-MAB}, by strong duality, we have
    \begin{align*}
        \E_{a \sim d_{w_\alpha^\star}} [r(a)] - \alpha \E_{a \sim \mu}[f(w^\star_\alpha(a))] \geq  \E_{a \sim d^\star} [r(a)] - \alpha \E_{a \sim \mu}[f(w^\star(a))]
    \end{align*}
    where $d_{w_\alpha^\star}(a) = \mu(a)w_\alpha^\star(a)$ by definition \eqref{eq:definitions_dw_sa} and we used the fact that $\E_{a \sim \mu}[w^\star_\alpha(a)] - 1 = \E_{a \sim \mu}[w^\star(a)] - 1 = 0$. Therefore, the suboptimality of $\pi^\star_\alpha$ can be bounded as follows
    \begin{align*}
        J(\pi^\star) - J(\pi^\star_\alpha) & = \E_{a \sim d^\star} [r(a)] 
        - \E_{a \sim d_{w_\alpha^\star}} [r(a)]\\
        & \leq \alpha \E_{a \sim \mu}[f(w^\star(a))] - \alpha \E_{a \sim \mu}[f(w^\star_\alpha(a))]\\
        & \leq \alpha \E_{a \sim \mu}[f(w^\star(a))] \leq \alpha f(C^\star) \leq \alpha B_f,
    \end{align*}
    where in the second to last inequality we used the non-negativity of $f$ and in the last equality, we used the boundedness of $f$.
\end{proof}

\subsection{Proof of Proposition \ref{prop:MAB-alpha-0-fails}}\label{app:MAB-alpha-0-fails}
Consider a 2-armed bandit instance with the following reward distributions, data distribution, and function classes.
\begin{itemize}
    \item \textit{Reward distributions:} The first arm is optimal with deterministic reward and the second arm has a Bernoulli distribution:
    \begin{align*}
        r(1) & = \frac{1}{2} \quad \text{w.p. } 1, \qquad r(2) \sim \text{Bernoulli}(1/3).
    \end{align*}
    \item \textit{Data distribution:} We consider a scenario where most data are concentrated on the optimal arm:
    \begin{align*}
        \mu(1) = 1 - \frac{2}{N}, \qquad \mu(2) = \frac{2}{N}.
    \end{align*}
    Here, the single-policy concentrability coefficient is $C^\star = 1/\mu(1)$ and is finite for $N > 2$. Let $N(a)$ denote the number of samples on arm $a$. To obtain upper and lower bounds on $N(a)$, we resort to the following lemma, which is a direct consequence of the Chernoff bound for binomial variables.

    \begin{lemma}[Chernoff bounds, binomial]\label{lemma:N-lower-bound-MAB} $ $
    \begin{enumerate}[label=(\Roman*)]
        \item With probability at least $1-\exp( - N \mu(a) \delta_u^2 /(2+\delta_u))$, one has $N(a) \leq (1+\delta_u) N\mu(a)$ for any $\delta_u > 0$;
        \item With probability at least $1-\exp( - N \mu(a) \delta_l^2 /2)$, one has $(1-\delta_l) N \mu(a) \leq N(a)$ for any $0 < \delta_l < 1$.
    \end{enumerate}
    \end{lemma}
    We condition on the event that the number of samples on the second arm is between 1 and 5 which occurs with probability larger than $1 - \exp\left(-2 \cdot \frac{0.9^2}{2}\right) - \exp\left(-2\cdot \frac{1.95^2}{1.95+2}\right)  \geq  0.4$ due to Lemma \ref{lemma:N-lower-bound-MAB} when setting $\delta_l = 0.9$ and $\delta_u = 1.95$:
    \begin{align*}
        (1-0.9)\cdot N\mu(2) \leq N(2) \leq (1+1.95) \cdot N \mu(2) \quad \Rightarrow \quad 1 \leq N(2) \leq 5.
    \end{align*}
    \item \textit{Function classes:} Assume that $\cW = \{w_1 = (C^\star, 0), w_2 = (0, B_w)\}$ and $\cV = \{1/2\}$. By Lemma \ref{lemma:regularized-unregularized-closeness-MAB}, we have $v^\star_0 = r^\star = 1/2$. Therefore, the problem is realizable as $v^\star_0 \in \cV$ and $w^\star_0 = w^\star = (C^\star, 0) \in \cW$. Furthermore, notice that for the second candidate $w_2 = (0,B_w) \in \cW$, the normalization factor is small for a constant $B_w$ as $d_{w_2 } = \sum_a w_2(a) \mu(a) = 2B_w/N$. 
\end{itemize}
Consider the case where all $N(a)$ samples on the second arm observe a reward of 1, which happens with a probability of at least $\frac{1}{3^5}$ as we conditioned on the event that $1 \leq N(2) \leq 5$. We now compute $\hat w$ by solving the empirical objective \eqref{eq:MAB-empirical-objective} with $\alpha = 0$. Note that since $|\cV| = 1$, it suffices to compute $\hat w = \arg \max_{w \in \cW} \hat L^{\text{MAB}}_0(w, v = 1/2)$. We have
\begin{align*}
    \hat L_0^{\text{MAB}}(w_1, 1/2) & = \frac{N(1)}{N} \left[ C^\star \cdot \frac{1}{2} - \frac{1}{2} \left(C^\star - 1\right) \right] + \frac{N(2)}{2N} = \frac{1}{2}\\
    \hat L_0^{\text{MAB}}(w_2, 1/2) & = \frac{N(1)}{2N} + \frac{N(2)}{N} \left[ B_w - \frac{1}{2}(B_w - 1)\right] = \frac{1}{2} + \frac{N(2) B_w}{2N}
\end{align*}
Since we conditioned on the event with $N(2) \geq 1$, solving the optimization problem $\max_{w \in \cW} \hat L^{\text{MAB}}_0(w, v = 1/2)$ finds $\hat w = (0,B_w)$, leading to a policy that picks the second arm with probability one. Therefore, with constant probability of $0.4 \times 1/3^5 > 0.001 $, we have 
\begin{align*}
    J(\pi^\star) - J(\hat \pi) = \frac{1}{2} - \frac{1}{3} = \frac{1}{6}.
\end{align*}

\subsection{Proof of Theorem \ref{thm:suboptimality-MAB}}\label{app:suboptimality-PRO-MAB}

Before embarking on the main proof, we present two lemmas related to the primal-dual regularized approach. The first lemma shows the closeness of population objective \eqref{eq:population-objective-MAB} to its empirical approximation used in Algorithm \ref{alg:PAL-MAB}, which is a direct consequence of Hoeffding's inequality. We also show that closeness of objectives results in the closeness of $w^\star_\alpha$ and $\hat w$, which are respectively the optimums to \eqref{eq:population-objective-MAB} and \eqref{eq:MAB-empirical-objective}. The proof of this lemma is deferred to the end of this subsection.

\begin{lemma}[Empirical and population closeness, PRO-MAB]\label{lemma:objective-statistical-error}
Fix $\delta > 0$ and define
\begin{align}\label{eq:statistical-error-MAB}
    \epsilon_{\text{stat}, \alpha}^{\text{MAB}} \coloneqq ((B_w + 1)(B_v + 1) + \alpha B_f) \sqrt{\frac{\log |\cV||\cW|/\delta}{N}}.
\end{align}
For any $w \in \cW$ and $v \in \cV$, the following bounds hold with probability at least $1-\delta$
\begin{enumerate}[label=(\Roman*)]
    \item\label{claim:stat-error-bound-MAB} $|L_\alpha^{\text{MAB}}(w, v) - \hat L_\alpha^{\text{MAB}}(w, v)| \leq \epsilon_{\text{stat}, \alpha}^{\text{MAB}}$;
    \item\label{claim:what-wstar-close-MAB} $L_\alpha^{\text{MAB}}(w^\star_\alpha, v) -  L_\alpha^{\text{MAB}}(\hat w, v) \leq 2 \epsilon_{\text{stat}, \alpha}^{\text{MAB}}$.
\end{enumerate}
\end{lemma}
The second lemma finds a lower bound on the occupancy normalization factor $d_{\hat w} = \sum_a \hat w(a) \mu(a)$ enforced by the behavior regularization.
\begin{lemma}[Occupancy validity enforced by behavior regularization]\label{lemma:log-convergence-hat-d-MAB} Let $f$ be an $M_f$-strongly-convex function and fix $\delta > 0$. Then, with probability at least $1-\delta$, one has 
\begin{align*}
    d_{\hat w} \geq 1 - \sqrt{\frac{4 \epsilon_{\text{stat}, \alpha}^{\text{MAB}}}{\alpha M_f}},
\end{align*}
where $\epsilon_{\text{stat}, \alpha}^{\text{MAB}}$ is defined in \eqref{eq:statistical-error-MAB}.
\end{lemma}
For the rest of this proof, we condition on the high probability events of Lemmas \ref{lemma:objective-statistical-error} and \ref{lemma:log-convergence-hat-d-MAB}. Define 
\begin{align}
    \epsilon_{\hat w, r} \coloneqq \sum_a w^\star_\alpha(a) \mu(a) r(a) - \hat w(a) \mu(a) r(a).
\end{align}
By part \ref{claim:what-wstar-close-MAB} of Lemma \ref{lemma:objective-statistical-error}, we have $L_\alpha^{\text{MAB}} (v^\star_\alpha, w^\star_\alpha) - L_\alpha^{\text{MAB}} (v^\star_\alpha, \hat w) \leq 2  \epsilon_{\text{stat}, \alpha}^{\text{MAB}}$. Therefore, 
\begin{align}\label{eq:converted-obj-diff-bound-MAB}
    \epsilon_{\hat w, r} - \alpha \E_{\mu} [f(w^\star_\alpha (a)) - f( \hat w (a))] + v^\star_\alpha  (d_{\hat w} - 1)  \leq 2 \epsilon_{\text{stat}, \alpha}^{\text{MAB}}.
\end{align}
Recall from Lemma \ref{lemma:regularized-unregularized-closeness-MAB} that we have $v^\star_\alpha = r^\star - \alpha c$, where $c \leq f'(C^\star)$. Thus, combined with \eqref{eq:converted-obj-diff-bound-MAB}, we write 
\begin{align}\notag 
    \epsilon_{\hat w, r} + r^\star  (d_{\hat w} - 1) & \leq 2  \epsilon_{\text{stat}, \alpha}^{\text{MAB}} + \alpha \E_{\mu} [f(w^\star_\alpha (a)) - f( \hat w (a))] + \alpha c (d_{\hat w} -1)\\\label{eq:bound-on-wdifference-constraint-violation-MAB}
    & \leq 2  \epsilon_{\text{stat}, \alpha}^{\text{MAB}} + \alpha(2  B_f + \alpha f'(C^\star) B_w),
\end{align}
where in the second line we used the bounds $|f(x)| \leq B_f$ and $d_{\hat w} \leq B_w$. Note that setting $\alpha = 16 \epsilon_{\text{stat}, 1}^{\text{MAB}}/ M_f$, Lemma \ref{lemma:log-convergence-hat-d-MAB} asserts that $d_{\hat w} \geq 1/2$. Since $d_{\hat w} \geq 1/2$, the learned policy is written as $\hat \pi = \hat w(a) \mu(a)/ d_{\hat w}$. With simple algebraic manipulations, we find the following expression for the suboptimality of $\hat \pi$ with respect to $\pi^\star_\alpha$: 
\begin{align*}
    J(\pi^\star_\alpha) - J(\hat\pi)
    & = \sum_a w^\star_\alpha(a) \mu(a) r(a)- \frac{1}{d_{\hat w}} \hat w(a) \mu(a)r(a)\\
    & = \sum_a w^\star_\alpha(a) \mu(a) r(a) - \hat w(a) \mu(a) r(a) + \sum_a \left(1- \frac{1}{d_{\hat w}}\right) \hat w(a) \mu(a) r(a)\\
    & = \epsilon_{\hat w, r} + \left(d_{\hat w}- 1\right) \sum_a \frac{1}{d_{\hat w}} \hat w(a) \mu(a) r(a)\\
    &  = \epsilon_{\hat w, r} + \left(d_{\hat w}- 1\right) J(\hat \pi) \\
    &  = \epsilon_{\hat w, r} + \left(d_{\hat w}- 1\right) J(\pi^\star_\alpha) - \left(d_{\hat w}- 1\right) \left[ J(\pi^\star_\alpha) - J(\hat \pi) \right].
\end{align*}
Let $\epsilon_{\text{reg}} = J(\pi^\star) - J(\pi^\star_\alpha) =  r^\star - J(\pi^\star_\alpha)$ denote the suboptimality suffered due to behavior regularization. Suboptimality $J(\pi^\star_\alpha) - J(\hat\pi)$ can be expressed as
\begin{align*}
    J(\pi^\star_\alpha) - J(\hat \pi) & = \frac{1}{d_{\hat w}} \left(\epsilon_{\hat w, r} + \left(d_{\hat w}- 1\right) J(\pi^\star_\alpha) \right)\\
    & = \frac{1}{\hat d} \left(\epsilon_{\hat w, r} + \left(d_{\hat w}- 1\right) (r^\star - \epsilon_{\text{reg}}) \right)\\
    & \leq \frac{1}{d_{\hat w}} \left(\epsilon_{\hat w, r} + \left(d_{\hat w}- 1\right) r^\star\right) - \frac{1}{d_{\hat w}} \left(d_{\hat w}- 1\right) \epsilon_{\text{reg}}.
\end{align*}
We use the above inequality to bound the suboptimality with respect to the optimal policy:
\begin{align*}
    J(\pi^\star) - J(\hat \pi) & = J(\pi^\star) - J(\pi^\star_\alpha) + J(\pi^\star_\alpha) - J(\hat \pi)\\
    & = \epsilon_{\text{reg}} + J(\pi^\star_\alpha) - J(\hat \pi)\\
    & \leq \epsilon_{\text{reg}} + \frac{1}{d_{\hat w}} \left(\epsilon_{\hat w, r} + \left(d_{\hat w}- 1\right) r^\star\right) - \frac{1}{d_{\hat w}} \left(d_{\hat w}- 1\right) \epsilon_{\text{reg}}\\
    & \leq \frac{1}{d_{\hat w}} \left(\epsilon_{\hat w, r} + \left(d_{\hat w}- 1\right) r^\star\right) + \frac{1}{d_{\hat w}} \epsilon_{\text{reg}}.
\end{align*}
Recall that we have $1/d_{\hat w} \leq 2$ and that $\epsilon_{\text{reg}}$ is bounded by $\alpha B_f$ by Lemma \ref{lemma:regularized-unregularized-closeness-MAB}. Therefore,
\begin{align*}
    J(\pi^\star) - J(\hat \pi) & \leq \frac{1}{d_{\hat w}} \left(\epsilon_{\hat w, r} + \left(d_{\hat w}- 1\right) r^\star\right) + \frac{1}{d_{\hat w}} \epsilon_{\text{reg}}\\
    & \leq 2 \left(\epsilon_{\hat w, r} + \left(d_{\hat w}- 1\right) r^\star\right) + 2 \alpha B_f\\
    & \leq 4 \epsilon_{\text{stat}, \alpha}^{\text{MAB}} + \alpha (4B_f + 2f'(C^\star) B_w)) + 2 \alpha B_f\\
    & \lesssim \alpha (B_f + f'(C^\star) B_w).
\end{align*}
where the penultimate inequality relies on the bound derived in \eqref{eq:bound-on-wdifference-constraint-violation-MAB}.

\begin{proof}[Proof of Lemma \ref{lemma:objective-statistical-error}]
$\hat L_\alpha^{\text{MAB}}(w, v)$ is an empirical average over independent and bounded random variables, where the bound on individual variables is computed as 
    \begin{align*}
        |w(a_i) r_i - \alpha f(w(a_i)) - v(w(a_i) -1)| & \leq B_w + \alpha B_f + B_v (B_w +1)\\
        & \leq (B_w + 1)(B_v + 1) + \alpha B_f.
    \end{align*}
    It is easy to see that $\E_{\cD} [\hat L_\alpha^{\text{MAB}}(w, v)] = L_\alpha^{\text{MAB}}(w,v)$, where the expectation is taken with respect to the randomness in dataset $\cD$. Part \ref{claim:stat-error-bound-MAB} of this lemma is proved by applying Hoeffding's inequality along with a union bound on $w$ and $v$.

    The proof of part \ref{claim:what-wstar-close-MAB} is similar to Lemma 7 of \citet{zhan2022offline} and relies on decomposing the objective difference and using the fact that $(\hat w, \hat v)$ correspond to the saddle points of $L_\alpha^{\text{MAB}}$ and $\hat L_\alpha^{\text{MAB}}$. For any $w \in \cW$, define 
    \begin{align}
        \hat v_w = \arg \min_{v \in \cV} \hat L_\alpha^{\text{MAB}}(w, v)
    \end{align}
    We write
    \begin{align*}
    L_\alpha^{\text{MAB}}(w^\star_\alpha, v) -  L_\alpha^{\text{MAB}}(\hat w, v) & = \underbrace{L_\alpha^{\text{MAB}}(w^\star_\alpha, v) - L_\alpha^{\text{MAB}}(w^\star_\alpha, \hat v_{w^\star_\alpha})}_{\coloneqq T_1} + \underbrace{L_\alpha^{\text{MAB}}(w^\star_\alpha, \hat v_{w^\star_\alpha}) - \hat L_\alpha^{\text{MAB}}(w^\star_\alpha, \hat v_{w^\star_\alpha})}_{\coloneqq T_2}\\
    & + \underbrace{\hat L_\alpha^{\text{MAB}}(w^\star_\alpha, \hat v_{w^\star_\alpha}) - \hat L_\alpha^{\text{MAB}}(\hat w, \hat v)}_{\coloneqq T_3} + \underbrace{\hat L_\alpha^{\text{MAB}}(\hat w, \hat v) - \hat L_\alpha^{\text{MAB}}(\hat w, v)}_{\coloneqq T_4}\\
    & + \underbrace{\hat L_\alpha^{\text{MAB}}(\hat w, v) - L_\alpha^{\text{MAB}}(\hat w, v)}_{\coloneqq T_5},
\end{align*}
Each term is bounded as follows:
\begin{itemize}
    \item $T_1 = 0$ because $w^\star_\alpha$ satisfies the constraint $\sum_a w^\star_\alpha(a) \mu(a) = 1$ and for any $v_1, v_2$ we have $L_\alpha^{\text{MAB}}(w^\star_\alpha, v_1) = L_\alpha^{\text{MAB}}(w^\star_\alpha, v_2)$.
    \item $T_2 \leq \epsilon_{\text{stat}}$ due to Lemma \ref{lemma:objective-statistical-error}.
    \item $T_3 \leq 0$ because $\hat w = \arg \max_{w \in \cW} \hat L_{\alpha}(\hat v_w, w)$.
    \item $T_4 \leq 0$ because $\hat v = \arg \min_{v \in \cV} \hat L_\alpha^{\text{MAB}}(v, \hat w)$.
    \item $T_5 \leq \epsilon_{\text{stat}}$ due to Lemma \ref{lemma:objective-statistical-error}.
\end{itemize}
Summing up the bounds on each term yields the desired bound.
\end{proof}

\begin{proof}[Proof of Lemma \ref{lemma:log-convergence-hat-d-MAB}]
    This lemma is a direct consequence of Lemma 8 in \citet{zhan2022offline}. For completeness, we present a simplified proof for the multi-armed bandit setting.

    First, observe that since $f$ is $M_f$-strongly-convex, the function $L_\alpha^{\text{MAB}}(v^\star_\alpha, w)$ is $\alpha M_f$-strongly-concave with respect to $w$ and norm $\|\cdot \|_{2, \mu}$. Furthermore, since $w^\star_\alpha = \arg \max_w L_\alpha^{\text{MAB}} (v^\star, w)$, we have 
    \begin{align*}
        \|\hat w - w^\star_\alpha\|_{2, \mu} \leq \sqrt{\frac{2(L_\alpha^{\text{MAB}}(w^\star_\alpha,v^\star_\alpha) - L_\alpha^{\text{MAB}} (\hat w,v^\star_\alpha))}{\alpha M_f}}.
    \end{align*}
    The above bound along with the bound on $L_\alpha^{\text{MAB}}(w^\star_\alpha,v^\star_\alpha) - L_\alpha^{\text{MAB}} (\hat w,v^\star_\alpha) \leq 2 \epsilon_{\text{stat}, \alpha}^{\text{MAB}}$ showed in Lemma \ref{lemma:regularized-unregularized-closeness-MAB}, give the following bound on $|d_{\hat w} - 1|$
    \begin{align*}
        \left|d_{\hat w} - 1\right| = \left|\sum_a \hat w(a) \mu(a) - \sum_a w^\star_\alpha(a) \mu(a) \right| \leq \|\hat w - w^\star_\alpha\|_{1, \mu} \leq  \|\hat w - w^\star_\alpha\|_{2, \mu} \leq \sqrt{\frac{4 \epsilon_{\text{stat}, \alpha}^{\text{MAB}}}{\alpha M_f}},
    \end{align*}
    which completes the proof.
\end{proof}

\subsection{Proof of Proposition \ref{prop:MAB-constraint-sufficient}}\label{app:MAB-constraint-sufficient}
Consider the difference between population objective with $\alpha = 0$ at $w^\star_0 = w^\star$ and $\hat w$, which is bounded by Lemma \ref{lemma:objective-statistical-error}:
\begin{align}\label{eq:objective-stat-error-step1}
    L(w^\star, v^\star) - L(\hat w, v^\star) & = \E_{a \sim \mu} [r(a) (w^\star(a) - \hat w(a))] - v^\star \E_{a \sim \mu} [w^\star(a) - \hat w(a)] \lesssim \epsilon_{\text{stat}, \alpha}^{\text{MAB}}.
\end{align}
We have $\E_{a \sim \mu}[w^\star(a)] = 1$ due to realizability and $\E_{a \sim \mu} [\hat w(a)] = 1$ is our assumption. Thus the second term in \eqref{eq:objective-stat-error-step1} is zero. Moreover, note that $\hat \pi(a) = \hat w(a) \mu(a)/\E_{a \sim \mu}[w(a)] = \hat w(a)$. Substituting the expression for $\epsilon_{\text{stat}, \alpha}^{\text{MAB}}$ from \eqref{eq:statistical-error-MAB} with $\alpha =0$, we obtain 
\begin{align*}
    J(\pi^\star) - J(\hat \pi) & = \E_{a \sim \mu} [r(a) (w^\star(a) - \hat w(a))] \lesssim B_w(B_v + 1) \sqrt{\frac{\log |\cV||\cW|/\delta}{N}},
\end{align*}
where we used the fact that $B_w \asymp B_w + 1$ since $B_w \geq 1$ due to realizability of $w^\star$.

\section{Proofs for contextual bandits}

This section of the appendix is organized as follows. In Appendix \ref{app:PRO-CB}, we present details of the PRO-CB algorithm. Appendix \ref{app:PRO-CB-suboptimal} is devoted the proof of Proposition \ref{prop:MAB-alpha-0-fails}, which shows that the PRO-CB algorithm fails to achieve statistically optimal rate of $1/\sqrt{N}$. The proof of suboptimality upper bound for the conservative offline CB algorithm with ALM is presented in Theorem \ref{thm:PAL-CB-suboptimality}.

\subsection{Primal-dual regularized offline contextual bandits (PRO-CB)}\label{app:PRO-CB}

Define importance weights $w(s,a) = d(s,a)/\mu(s,a)$ to denote the ratio of occupancy and data distribution. The primal-dual regularized approach \cite{zhan2022offline} solves the following population objective\small 
\begin{align}\label{eq:population-objective-CB}
    & \max_{w \geq 0} \min_{v}  L_\alpha^{\text{CB}}(w, v) \coloneqq \E_{s,a \sim \mu} \left[ w(s,a)r(s,a) \right]  - \E_{s, a \sim \mu}[v(s) (w(s,a) - 1)] - \alpha \E_{s, a \sim \mu} \left[ f\left(w(s,a)\right) \right],
\end{align}\normalsize
The above optimization problem satisfies strong duality. We define $w^\star_\alpha, v^\star_\alpha$ to respectively denote the optimal solutions to the primal and dual variables. Approximating $w, v$ to belong to function classes $\cW, \cV$ and solving the empirical version of objective \eqref{eq:population-objective-CB} leads to the PRO-CB given in Algorithm \ref{alg:PRO-CB}.

\begin{algorithm}[h]
\caption{Primal-dual Regularized Offline Contextual Bandits (PRO-CB)}
\label{alg:PRO-CB}
\begin{algorithmic}[1]
\State \textbf{Inputs:} Dataset $\cD = \{(s_i, a_i, r_i)\}_{i=1}^N$, function classes $\cW, \cV$, function $f(\cdot)$, parameter $\alpha$
\State Find a solution $\hat w, \hat v$ to the following problem
\begin{align}
    \begin{split}\label{eq:PRO-CB-empirical-objective}
        \max_{w \in \cW} \min_{v \in \cV} \hat L_\alpha^{\text{CB}}(w, v) \coloneqq \frac{1}{N}\sum_{i=1}^N w(s_i, a_i) r_i - \alpha f(w(s_i, a_i)) - v(s_i) (w(s_i, a_i) - 1). 
    \end{split}
\end{align}
\State \textbf{Return:} $\hat \pi = \pi_{\hat w}$. 
\end{algorithmic}
\end{algorithm}

\subsection{Proof of Proposition \ref{prop:PRO-CB-suboptimal}}\label{app:PRO-CB-suboptimal}

 We separate the proof into two cases: $\alpha \geq N^{\beta}$ for $\beta > -1/2$ and $\alpha \leq \widetilde{O}( N^{-1/2})$. When $\alpha$ is large, we show that the large bias caused by regularization results in suboptimality of $\alpha$ even in MABs. When $\alpha$ is small, we construct a two-state CB instance (as the single-state case is indeed successful due to Theorem \ref{thm:suboptimality-MAB}), showing that such small $\alpha$ does not sufficiently enforce occupancy validity in states with a relatively small but still significant state distribution $\rho(s)$.

\subsubsection{Proof for large $\alpha$}
\label{sec:prop3-large-alpha}
If there exists $-\frac{1}{2}<\beta$ such that $\alpha\ge N^\beta$, then we consider a simple single-state two-arm contextual bandit (equivalently multi-armed bandit) instance:
\begin{itemize}
    \item \emph{Reward distribution}: Both arms have deterministic rewards and the suboptimal arm has a value gap of $\alpha$:
    \begin{align*}
        r(1) = 1 \quad \text{w.p. }1,\quad r(2) = \max\{0,1-\alpha\} \quad \text{w.p. }1.
    \end{align*}
    
    \item \emph{Data distribution}: We construct the data distribution such that both arms have constant probability density, which implies a constant concentrability ratio $C^\star$. Here we assume $M_f<100$ for convenience, but if $M_f$ is larger we can use the same construction with an even larger constant as the denominator. 
    \begin{align*}
        \mu(1)=\frac{M_f}{100},\quad\mu(2)=1-\frac{M_f}{100}.
    \end{align*}
    
    \item \emph{Function classes}: We assume both $\cW$ and $\cV$ contain only the optimal regularized solutions $(w^\star_\alpha, v^\star_\alpha)$ and the optimal unregularized solutions $(w^\star, v^\star)$, which satisfy the realizability requirements of PRO-CB:
    \begin{align*}
        \cW=\{w_\alpha^\star,w^\star\},\quad\cV=\{v_\alpha^\star,v^\star\}.
    \end{align*}
\end{itemize}
Our argument is broken down in two steps. In the first step, we show that the suboptimality of the optimal regularized policy, which is the policy induced by the regularized optimal weights $\pi^\star_\alpha \coloneqq \pi_{w^\star_\alpha}$, is at least of order $\min \{1, \alpha\}$. Then, in the second step, we prove that $w^\star_\alpha$ is chosen with a constant probability.

\paragraph{Step 1: Suboptimality of $\pi_\alpha^\star$.}
In the particular offline bandit instance above, we show the following lower bound on suboptimality of $\pi^\star_\alpha$
\begin{align}
    J(\pi^\star)-J(\pi_\alpha^\star)=\pi_\alpha^\star(2)\cdot(r(1)-r(2))
    =\mu(2) w_\alpha^\star(2) \cdot \min\{1,\alpha\} = \Omega( \min \{1, \alpha\}).\label{eq:tmp-suboptimality}
\end{align}
To establish \eqref{eq:tmp-suboptimality}, we show that $ w_\alpha^\star(2)> c$ for a fixed constant $c=\frac{1}{2}$. We prove this by contradiction. Suppose
\begin{align}
    w_\alpha^\star(2)\le c \label{eq:assump-prop3}.
\end{align}
By KKT conditions we have 
\begin{align*}
    w_\alpha^\star(2) & = \max \left\{0, (f')^{-1}\left(\frac{r(2)-v_\alpha^\star}{\alpha}\right)\right\} \ge (f')^{-1}\left(\frac{r(2)-v_\alpha^\star}{\alpha}\right).
\end{align*}
Therefore, using the fact that $f'$ is strictly increasing since $f$ is strictly convex, we lower bound $v^\star_\alpha$ according to
\begin{align*}
    v_\alpha^\star\ge r(2)-\alpha f'(w_\alpha^\star(2))\ge
    r(1)-(r(1)-r(2))-\alpha f'\left(c\right).
\end{align*}
Combining the above bound on $v_\alpha^\star$ with the KKT condition on $w_\alpha^\star(1)$, we then obtain
\begin{align}
    w_\alpha^\star(1)=&(f')^{-1}\left(\frac{r(1)-v_\alpha^\star}{\alpha}\right)
    \le (f')^{-1}\left(\frac{r(1)-r(2)}{\alpha}+f'\left(c\right)\right).
    \label{eq:tmp2-prop3}
\end{align}
Here, we used the fact that $v^\star_\alpha \geq r^\star = r(1)$ and that $f(0) = 0$ so $(f')^{-1}((r(1) - v^\star_\alpha) /\alpha) \geq 0$. 
Moreover, since the regularization function $f$ is $M_f$-strongly convex, we write
\begin{align}
    &f'\left(\frac{1-c\mu(2)}{\mu(1)}\right)-f'\left(c\right)
    \ge M_f\left(\frac{1-c\mu(2)}{\mu(1)}-c\right)
    =M_f\frac{1-c}{\mu(1)}=100(1-c)>1,\nonumber\\
    \Rightarrow\ & 
    \frac{r(1)-r(2)}{\alpha}+f'\left(c\right)\le
    1+f'\left(c\right)< f'\left(\frac{1-c\mu(2)}{\mu(1)}\right).\label{ineq:key}
\end{align}
Therefore, we can continue to upper bound the RHS of \eqref{eq:tmp2-prop3}:
\begin{align*}
    w_\alpha^\star(1)\le (f')^{-1}\left(\frac{r(1)-r(2)}{\alpha}+f'\left(c\right)\right)
    \underbrace{<}_{\text{by \eqref{ineq:key}}}(f')^{-1}\left(f'\left(\frac{1-c\mu(2)}{\mu(1)}\right)\right)=\frac{1-c\mu(2)}{\mu(1)},
\end{align*}
which further implies that
\begin{align}
    w_\alpha^\star(1)\mu(1)< 1-c\mu(2)
    \underbrace{\le}_{\text{ by \eqref{eq:assump-prop3}}} 1-w_\alpha^\star(2)\mu(2)\ \Rightarrow\ 
    \sum_a w_\alpha^\star(a)\mu(a)<1.\label{eq:prop3-case1}
\end{align}

Note that \eqref{eq:prop3-case1} contradicts with the fact that $(w_\alpha^\star,v_\alpha^\star)$ is the optimal min-max solution of $L_\alpha^{\text{MAB}}$ because it violates the constraint $\E_{\mu}[w(a)]=1$. Therefore, 
\eqref{eq:assump-prop3} should not hold in the first place, and
we must have
\begin{align}
    J(\pi^\star)-J(\pi_\alpha^\star)=\mu(2) w_\alpha^\star(2) \cdot (r(1)-r(2)) >c\left(1-\frac{M_f}{100}\right) \min\{1,\alpha\} \gtrsim \min\{1,\alpha\}
    \label{eq-sub-gap}
\end{align}

\paragraph{Step 2: $w_\alpha^\star$ is picked with large probability.}
We now show that $w_\alpha^\star$ is picked by the algorithm with at least a constant probability. Note that since $w_\alpha^\star$ and $w^\star$ both satisfy the constraint $\E_{\mu}[w]-1=0$, objectives $L_\alpha^{\text{MAB}}(w_\alpha^\star,v)$ and $L_\alpha^{\text{MAB}}(w_\alpha,v)$ do not depend on the Lagrange multiplier variable $v$. We argue that at the population level, we have the following lower bound on the gap  $L_\alpha^{\text{MAB}}(w_\alpha^\star,v)-L_\alpha^{\text{MAB}}(w^\star,v) \gtrsim \alpha $.
Using the definition of $L_{\alpha}^{\text{MAB}}$, one has 
\begin{align}
    &L_\alpha^{\text{MAB}}(w_\alpha^\star,\cdot)-L_\alpha^{\text{MAB}}(w^\star,\cdot )\nonumber\\
    =& \alpha \E_{\mu}[f(w^\star(a))-f(w_\alpha^\star(a))]-\mu(2) w_\alpha^\star(2) (r(1)-r(2))\nonumber\\
    =& \alpha\Bigg(
    \mu(1)\Big(f(w^\star(1))-f(w_\alpha^\star(1))\Big)+\mu(2)\Big(f(w^\star(2))-f(w_\alpha^\star(2))\Big)
    \Bigg)-\mu(2) w_\alpha^\star(2) (r(1)-r(2))
    \nonumber\\
    \ge&\alpha\left(
    \mu(1)\Big(w^\star(1)-w_\alpha^\star(1)\Big)\cdot f'(w_\alpha^\star(1))
    -\mu(2)f(w_\alpha^\star(2))
    \right)-\mu(2) w_\alpha^\star(2) (r(1)-r(2))\label{tmpeq:cvx}\\
    =&\alpha\mu(2)\left(
    f'(w_\alpha^\star(1))\cdot w_\alpha^\star(2)
    -f(w_\alpha^\star(2))-w_\alpha^\star(2)\cdot\frac{r(1)-r(2)}{\alpha}
    \right),\label{tmpeq3}
\end{align}
In \eqref{tmpeq:cvx}, we used the convexity of regularization function $f$ as well as the fact that $f(w^\star(2)) = f(0) = 0$. Moreover, \eqref{tmpeq3} holds because
\begin{align*}
    \mu(1)\left(w^\star(1)-w_\alpha^\star(1)\right)=
\mu(1)\left(\frac{1}{\mu(1)}-w_\alpha^\star(1)\right)
=1-\mu(1)w_\alpha^\star(1)=\mu(2)w_\alpha^\star(2).
\end{align*}
By KKT conditions we also have
\begin{align}
    f'(w_\alpha^\star(1))=\frac{r(1)-v_\alpha^\star}{\alpha}
    =\frac{r(1)-r(2)+\alpha f'(w_\alpha^\star(2))}{\alpha}=\frac{r(1)-r(2)}{\alpha}+f'(w_\alpha^\star(2)).\label{tmpeq2}
\end{align}
Plugging \eqref{tmpeq2} back into  \eqref{tmpeq3}, we obtain
\begin{align}
    L_\alpha^{\text{MAB}}(w_\alpha^\star,\cdot)-L_\alpha^{\text{MAB}}(w^\star,\cdot)\ge&\alpha\mu(2)\Big(
    f'(w_\alpha^\star(2))\cdot w_\alpha^\star(2)
    -f(w_\alpha^\star(2))
    \Big)\nonumber\\
    \ge&\alpha\mu(2)\cdot\frac{M_f}{2}w_\alpha^\star(2)^2 >\alpha\mu(2)\cdot\frac{M_f}{2} c^2 \gtrsim \alpha,\label{tmpeq4}
\end{align}
where \eqref{tmpeq4} is based on the fact that $f$ is $M_f$-strongly convex, and that $w_\alpha^\star(2)>c$ proved in Step 1. We now prove that such large lower bound on population objective difference leads the algorithm to select $w^\star_\alpha$. Recall from Lemma \ref{lemma:objective-statistical-error} that with at least constant probability (e.g. setting $\delta = 0.1$), for any $v \in \cV, w \in \cW$, one has the following bound on difference between the population and empirical objectives 
\begin{align*}
    \left| L_{\alpha}^{\text{MAB}}(w,v) - \hat L_{\alpha}^{\text{MAB}}(w,v) \right| \lesssim 2 \epsilon_{\text{stat},\alpha}^{\text{MAB}}, 
\end{align*}
where $\epsilon_{\text{stat},\alpha}^{\text{MAB}}$ is of order $1/\sqrt{N}$ as defined in \eqref{eq:statistical-error-MAB}. Combining the above inequality with \eqref{tmpeq4}, for any $v,v'\in\cV$ we have
\begin{align*}
    & \hat{L}_\alpha^{\text{MAB}}(w_\alpha^\star,v)-\hat{L}_\alpha^{\text{MAB}}(w^\star,v')\\
    & \gtrsim \alpha - \epsilon_{\text{stat},\alpha}^{\text{MAB}} \gtrsim
    \alpha-(1+\alpha)N^{-\frac{1}{2}}\gtrsim N^\beta-N^{-\frac{1}{2}}.
\end{align*}
Therefore, since $\beta > -1/2$, we conclude that $w^\star_\alpha$ is chosen by the algorithm with constant probability:
\begin{align*}
    \min_{v\in\cV}\hat{L}_\alpha^{\text{MAB}}(w_\alpha^\star,v)-\min_{v\in\cV}\hat{L}_\alpha^{\text{MAB}}(w^\star,v)>0\ \Rightarrow\ 
    w_\alpha^\star=\argmax_{w\in\cW}\min_{v\in\cV}\hat{L}_\alpha^{\text{MAB}}(w,v).
\end{align*}
Combining the above result with the suboptimality lower bound of $\pi_\alpha^\star$ in \eqref{eq-sub-gap} completes the proof for $\alpha \geq N^\beta$.
\subsubsection{Proof for small $\alpha$}
Now suppose $\alpha\le \widetilde{O}( N^{-\frac{1}{2}})$, where $\widetilde{O}$ hides the logarithmic factors. In this case, we consider the following two-state two-arm contextual bandit instance:
\begin{itemize}
    \item \emph{State and reward distributions}: We construct the states 
    such that
    state 1 has a very small probability mass. For state 1, the first arm is optimal with a Bernoulli-distributed reward and the second arm is suboptimal with a deterministic reward. For state 2, both arms have deterministic rewards. Importantly, state 1 has a constant value gap in its suboptimal action.
    \begin{align*}
        &\rho(1)=N^{-\frac{1}{4}},\quad r(1,1)\sim\text{Bernoulli} \left(\frac{1}{2} \right),\ r(1,2)\equiv\frac{1}{3};\\
        &\rho(2)=1-N^{-\frac{1}{4}},\quad r(2,1)\equiv\frac{1}{2},\  r(2,2)\equiv\frac{1}{3}.
    \end{align*}
    \item \emph{Data distribution}: We assume that for both states, most of the probability density is concentrated on the optimal arm.
    \begin{align*}
        &\mu(s)=\rho(s),\ s=1,2.\\
        &\mu(1|1)=\mu(1|2)=1-\frac{2}{N},\quad \mu(2|1)=\mu(2|2)=\frac{2}{N}.
    \end{align*}
    \item \emph{Function classes}: Let $w$ be defined as $\tilde w(2,a)=w_\alpha^\star(2,a)$ and $\tilde w(1,a)=0$ for $a=1,2$. Consider the following function classes $\cW$ and $\cV$:
    \begin{align}\label{eq:W-V-class-constructions}
        \cW=\{w_\alpha^\star,\tilde w\},\ \cV=\{v_\alpha^\star,v^\star\}.
    \end{align}
\end{itemize}

The proof is broken down into 4 steps. In the first step, we show that when $\alpha\le\widetilde{O}( N^{-\frac{1}{2}})$ and $N$ is sufficiently large, the regularized optimal policy is the same as the unregularized optimal policy, i.e., $w_\alpha^\star=w^\star$. Therefore, the function class $\cW$ defined in \eqref{eq:W-V-class-constructions} is realizable $w^\star_\alpha = w^\star \in \cW$. In the second step, we prove that with constant probability $v^\star_\alpha = \argmin_{v \in \cV} \hat L^{\text{CB}}_\alpha (\tilde w,v)$. Then, we show that solving the saddle point of the empirical objective $\hat L_\alpha^{\text{CB}}(w,v)$ selects $\tilde w$ over $w^\star_\alpha$ with a constant probability. Finally, we prove that $\tilde w$ induces a policy $\pi_{\tilde w}$ that suffers from suboptimality of order $N^{-1/4}$, which completes the proof.

\paragraph{Step 1: Regularized optimal weights coincides with unregularized optimal weights.} Since the population optimization problem \eqref{eq:population-objective-CB} is independent across states at a population level, we can use the result of Lemma \ref{lemma:regularized-unregularized-closeness-MAB} to conclude that 
\begin{align*}
    &v_\alpha^\star(s)=r^\star(s)-c(s)\alpha,\text{ and }\\
    &w_\alpha^\star(s,a)=\max\left\{0,(f')^{-1}\left(\frac{r(s,a)-v_\alpha^\star(s)}{\alpha}\right)\right\}=\max\left\{0,(f')^{-1}\left(c(s)-\frac{r^\star(s)-r(s,a)}{\alpha}\right)\right\},
\end{align*}
where $0\le c(s)\le f'(C^\star)$ for $s\in\{1,2\}$. Since $r^\star(s)-r(s,2)=\frac{1}{6}=\Theta(1)$, for $N\ge (6f'(C^\star))^2$, we have $w_\alpha^\star(s,2)=0$ for the suboptimal arm $2$. Thus $w_\alpha^\star(s)=w^\star(s)=\frac{1}{\mu(1|s)}$. Correspondingly, we can use the KKT conditions to compute $v_\alpha^\star(s)=r^\star(s)-\alpha f'\left(\frac{1}{\mu(1|s)}\right)$.

\paragraph{Step 2: $v^\star_\alpha = \argmin_{v \in \cV} \hat L^{\text{CB}}_\alpha (\tilde w,v)$ with constant probability.}
Let $\hat{\mu}$ denote the empirical state-arm distribution and $\hat{r}$ denote the empirical mean reward. Define the following event:
\begin{align}\label{eq:event-def}
   \cE \coloneqq \left \{ \sum_a \hat{\mu}(a|s) \tilde w(s,a)\le 1 \; \text{ for } \; s\in\{1,2\} \right\}.
\end{align}
Recall that we defined $\tilde w(1,a) = 0$ and $\tilde w(2,a) = w^\star(2,a)$. Thus, the above event can be equivalently written as 
\begin{align}\label{tmpeq:equiv1}
    \sum_a \hat \mu(a|2) w^\star_\alpha(2, a) \leq 1 \iff \sum_a \left(\hat \mu(a|2) - \mu(a|2)\right)w^\star_\alpha(2, a) \leq 0.
\end{align}
Here we used the fact that $\sum_a \mu(a|2) w^\star_\alpha(s,2)=1$. Moreover, in Step 1 we showed that  $w^\star_\alpha = w^\star$, thus $w^\star_\alpha(2,2) = 0$ and \eqref{tmpeq:equiv1} corresponds to the following event
\begin{align}\label{tmpeq:equiv-form-final}
    \cE = \left \{ \hat \mu(1|2) - \mu (1|2) \leq 0 \right\}.
\end{align}
Since $\hat{\mu}(1|2)$ is an empirical version of the conditional probability $\mu(1|2)$, event $\cE$ happens with probability $\frac{1}{2}$. 

We condition on the event $\cE$ for the rest of the proof. Using the fact that $v_\alpha^\star(s) \le r^\star(s)= v^\star(s)$, we conclude that
\begin{align}
    \hat L_\alpha^{\text{CB}}(\tilde w, v^\star_\alpha)\le \hat L_\alpha^{\text{CB}}(\tilde w, v^\star)\ \Rightarrow\ 
    \hat L_\alpha^{\text{CB}}(\tilde w, v^\star_\alpha)=\min_{v\in\cV}\hat L_\alpha^{\text{CB}}(\tilde w, v).\label{eq:argmin=v-alpha}
\end{align}
\paragraph{Step 3: Analyzing the probability of picking $w_\alpha^\star$.}
Now we compare the value of $\hat L_\alpha^{\text{CB}}(\cdot,v_\alpha^\star)$ evaluated at $\tilde w$ and $w^\star_\alpha$. We use the definition $\tilde w(2, a) = w^\star_\alpha(2,a)$ and write
\begin{align}
    &\hat L_\alpha^{\text{CB}}(w_\alpha^\star, v^\star_\alpha) - \hat L_\alpha^{\text{CB}}(\tilde w, v^\star_\alpha)\nonumber\\
    = & \hat{\mu}(1) \Bigg[\hat{r}(1,1) \hat{\mu}(1|1)  w_\alpha^\star(1,1) 
    + \alpha 
     \hat \mu(1|1) f(w_\alpha^\star(1,1))
     + v_\alpha^\star(1) \left(\sum_{a} \hat\mu(a|1) \left( w(1,a)-w_\alpha^\star(1,a)\right)\right) \Bigg]\nonumber
\end{align}
Noting that $v_\alpha^\star(1)=r(1,1)-\alpha f'\left(w_\alpha^\star(s_1,a_1)\right)$, $\tilde w(1, a) = 0$, and $w_\alpha^\star(1,1)=w^\star(1,1)=\frac{1}{\mu(1|1)}$, we further simplify the above equation
\begin{align}
    &\hat L_\alpha^{\text{CB}}(w_\alpha^\star, v^\star_\alpha) - \hat L_\alpha^{\text{CB}}(\tilde w, v^\star_\alpha)\nonumber \\
    = & \hat{\mu}(1) \Bigg[\left(\hat{r}(1,1)-r(1,1)+\alpha f'\left(\frac{1}{\mu(1|1)}\right)\right)
     \frac{\hat\mu(1|1)}{\mu(1|1)}   
     + \alpha 
     \hat \mu(1|1) f\left(\frac{1}{\mu(1|1)}\right)
     \Bigg]\label{eq-exp-1}\\
     = & \hat{\mu}(1,1) \Bigg[
     \frac{\hat{r}(1,1)-r(1,1)}{\mu(1|1)}
     +\alpha\cdot
     \left( \frac{1}{\mu(1|1)} f'\left(\frac{1}{\mu(1|1)}\right)
     +f\left(\frac{1}{\mu(1|1)}\right)
     \right)
     \Bigg].
     \label{eq:case2-prop3}
\end{align}
We then prove that with constant probability, the first term in \eqref{eq:case2-prop3} is negative with a magnitude larger than the second term:
\begin{align}
    \frac{\hat{r}(1,1)-r(1,1)}{\mu(1|1)} \lesssim - N^{-3/8}.\label{eq:cond-case2-prop3}
\end{align}
The proof of this inequality relies on anti-concentration bounds of binomial random variables and is presented at the end of this section. By Inequality \eqref{eq:cond-case2-prop3} combined with \eqref{eq:argmin=v-alpha}, we conclude that 
\begin{align}
    \min_{v\in\cV}\hat L_\alpha^{\text{CB}}(\tilde w, v)= \hat L_\alpha^{\text{CB}}(\tilde w, v^\star_\alpha) > \hat L_\alpha^{\text{CB}}( w_\alpha^\star, v^\star_\alpha)
    \ge\min_{v\in\cV}\hat L_\alpha^{\text{CB}}( w_\alpha^\star, v),
\end{align}
which guarantees that the algorithm picks $\tilde w$ with a constant probability.

\paragraph{Step 4: Suboptimality of $\pi_w$}
Finally, for the policy $\pi_w$ induced by $w$, we have
\begin{align*}
    J(\pi^\star_\alpha)-J(\pi_w)=\mu(s_1)\pi_w(2|1) (r(1,1)-r(1,2))=\frac{N^{-\frac{1}{4}}}{12}\ge \Omega(N^\beta),
\end{align*}
for $\beta=-\frac14>-\frac12$, as desired. The proof for small $\alpha$ is thus complete.

\textit{Proof of Inequality \eqref{eq:cond-case2-prop3}.} Using the Chernoff bounds for binomial random variables given in Lemma \ref{lemma:N-lower-bound-MAB}, one can conclude that the following event $\cE'$ happens with probability at least 0.5:
\begin{align}
    \cE'\coloneqq\left\{N(1,1)\ge 0.1 N \mu(1,1)\ge 0.05 N^{\frac{3}{4}}\right\}.
\end{align}
Furthermore, $\cE$ and $\cE'$ are independent because the random variable $\hat{r}(s_1,a_1)$ is independent from the arm distribution within state $s_2$.
Therefore, conditioning on $\cE\cap\cE'$ which happens with probability $0.5\times0.5=0.25$, we use the anti-concentration bounds for Binomial random variables Lemma \ref{lemma:anti-concentration-binomial} to obtain the following lower bound:
\begin{align}
    \Pr\left.\left(\hat r(1,1) - r(1,1) \leq - \sqrt{\frac{\log(2c_1)}{c_2 N(1,1) }} \le - c' N^{-\frac{3}{8}}\right|\cE\cap\cE'\right) \geq 0.5,
\end{align}
where $c'=\sqrt{\frac{20\log(2c_1)}{c_2 }} $ is a universal constant. Therefore, we have established that \eqref{eq:cond-case2-prop3} holds with constant probability.

\begin{lemma}[Anti-concentration of Binomial random variables, adapted from Proposition 7.3.2 of \citet{matouvsek2001probabilistic}]
\label{lemma:anti-concentration-binomial} 
    Let $X_1,\cdots, X_n$ be independent random variables following the Bernoulli distribution with mean $\frac{1}{2}$, and let $\overline{X}=\frac{1}{n}\sum_{i=1}^n X_i $ be the empirical mean. Then we have that for any $t\in[0,\frac{1}{8}]$ and universal constants $c_1,c_2$,
    \begin{align}
        \Pr\left(\overline{X}\le\E[\overline{X}]-t\right)\ge c_1 e^{-c_2 t^2n}.
    \end{align}
\end{lemma}

\subsection{Proof of Theorem \ref{thm:PAL-CB-suboptimality}}\label{app:PAL-CB-suboptimality}

Proof of this theorem largely follows similar steps as the proof we presented for Theorem \ref{thm:suboptimality-MAB}. In particular, we start by presenting two lemmas. The first lemma leverages Hoeffding's inequality to establish the closeness of the population objective \eqref{eq:AL-population-objective-CB} and empirical objective \eqref{eq:CB-empirical-objective-AL}. Additionally, we show that this result leads to the closeness of population objective at $w^\star$ and $\hat w$. Proof of this lemma is presented at the end of this subsection. 

\begin{lemma}[Empirical and population closeness, CB]\label{lemma:objective-statistical-error-CB}
Fix $\delta > 0$ and define 
\begin{align}\label{eq:statistical-error-CB}
    \epsilon_{\text{stat}}^{\text{CB}} \coloneqq 3(B_w + 1)^2 (B_v+1) \sqrt{\frac{\log (|\cW||\cV|/\delta)}{N}}.
\end{align}
For any $w \in \cW$ and $v \in \cV$, the following statements hold with probability at least $1-\delta$
\begin{enumerate}[label=(\Roman*)]
    \item\label{claim:CB-empirical-popoulation-close} $\left|L_{\text{AL}}^{\text{CB}}(w,v) - \hat L_{\text{AL}}^{\text{CB}}(w,v) \right| \leq \epsilon_{\text{stat}}^{\text{CB}}$;
    \item\label{claim:CB-what-wstar-close} $L_{\text{AL}}^{\text{CB}}(w^\star,v) - L_{\text{AL}}^{\text{CB}}(\hat w,v) \leq 2\epsilon_{\text{stat}}^{\text{CB}}$.
\end{enumerate}
\end{lemma}

In the second lemma, we prove that the ALM term enforces a lower bound on normalization factors $\frac{d_{\hat w}}{\mu}(s) \coloneqq \sum_a \hat w(s,a) \mu(a|s)$ for significant states.

\begin{lemma}[Occupancy validity enforced by the ALM]\label{lemma:lowerbound-dhatw}
Define the state space subset 
\begin{align}\label{eq:Ss-def}
    \mathcal{S}_s \coloneqq \left\{ s \ \middle| \  \frac{d_{\hat w}}{\mu}(s) \leq \frac{1}{2} \right\}.
\end{align}
For any fixed $\delta > 0$, the following statements hold with probability at least $1-\delta$,
\begin{enumerate}[label = (\Roman*)]
    \item\label{claim:avg-over-dhat-small}  $\E_{s, a \sim \mu} \left[ (r^\star(s) - r(s,a)) \hat w(s,a)\right] \lesssim \epsilon^{\text{CB}}_{\text{stat}}$;
    \item\label{claim:mass-on-invalid-small} $\sum_{s \in \cS_s} \mu(s) \lesssim \epsilon^{\text{CB}}_{\text{stat}}$; 
\end{enumerate}
where $\epsilon^{\text{CB}}_{\text{stat}}$ is defined in \eqref{eq:statistical-error-CB}.
\end{lemma}

Given the two lemmas above, our suboptimality analysis can be broken down into two simple steps. First, we partition the states based on $\cS_{s}$ defined in \eqref{eq:Ss-def} and decompose the policy suboptimality accordingly:
\begin{align}\notag
    \sum_s \mu(s) V^\star(s) - \sum_s \mu(s) V^{\hat \pi}(s)
    & = \sum_{s \in \cS_{\text{s}}} \mu(s) (V^\star(s) - V^{\hat \pi}(s)) + \sum_{s \not \in \cS_{\text{s}} } \mu(s) (V^\star(s) - V^{\hat \pi}(s)),\\ \label{eq:partition-bound-CB} 
    & \lesssim \epsilon_{\text{stat}}^{\text{CB}} + \sum_{s \not \in \cS_{\text{s}} } \mu(s) (V^\star(s) - V^{\hat \pi}(s))\\ \label{eq:bound-dhatw-inSl-CB}
    & \leq \epsilon_{\text{stat}}^{\text{CB}} + 2 \sum_s {d_{\hat w}(s)} (V^\star(s) - V^{\hat \pi}(s))
\end{align}
In \eqref{eq:partition-bound-CB}, we used part \ref{claim:mass-on-invalid-small} in Lemma \ref{lemma:lowerbound-dhatw} to bound the first term and \eqref{eq:bound-dhatw-inSl-CB} uses the fact that by definition, for all $s \not \in \cS_{\text{s}}$ we have $\mu(s) < 2\hat d(s)$ and $V^\star(s) - V^{\hat \pi}(s) \geq 0$. Moreover, the second term in \eqref{eq:bound-dhatw-inSl-CB} is bounded by part \ref{claim:avg-over-dhat-small} of Lemma \ref{lemma:lowerbound-dhatw} since
\begin{align*}
    \sum_s d_{\hat w}(s) (V^\star(s) - V^{\hat \pi}(s)) & = \sum_{s: d_{\hat w}(s) > 0} d_{\hat w}(s) \left(r^\star(s) - \sum_a \hat \pi(a|s) r(s,a) \right)\\
    & = \sum_{s: d_{\hat w}(s) > 0} \sum_a d_{\hat w}(s,a) r^\star(s) - d_{\hat w}(s) \frac{\hat w(s,a)\mu(s,a)}{d_{\hat w}(s)} r(s,a)\\
    & = \sum_{s: d_{\hat w}(s) > 0} \sum_a {\hat w}(s,a)\mu(s,a) r^\star(s) - \hat w(s,a) \mu(s,a) r(s,a)\\
    & \leq \sum_{s,a} \mu(s,a) {\hat w}(s,a) (r^\star(s) - r(s,a)) \lesssim \epsilon^{\text{CB}}_{\text{stat}},
\end{align*}
where the equations follow from the definition of $\hat \pi$. The final suboptimality bound is proved by noting that $(B_w +1)^2 \asymp B_w^2$ since $B_w \geq 1$ due to realizability of $w^\star(s,a^\star) \geq 1$.

\begin{proof}[Proof of Lemma \ref{lemma:objective-statistical-error-CB}]
To prove part \ref{claim:CB-empirical-popoulation-close}, notice that $\E_\mu \left[\hat L^{\text{CB}}_{\text{AL}}(w, v)\right] = L^{\text{CB}}_{\text{AL}}(w, v)$. Furthermore, $\hat L^{\text{CB}}_{\text{AL}}(w, v)$ is an empirical average of i.i.d.~random variables which are bounded by 
\begin{align*}
    & \left|w(s,a) r(s,a) - v(s) (w(s,a) - 1) - \left( \sum_a w(s,a) \mu(a|s) - 1\right)^2 \right|\\
    & \quad \leq B_w + B_v(B_w+1) + B_w^2\\
    & \quad \leq 3(B_w+1)^2 (B_v+1)
\end{align*}
Applying Hoeffding's inequality along with a union bound on $w \in \cW$ and $v \in \cV$ finishes the proof of part \ref{claim:CB-empirical-popoulation-close}.

We now prove part \ref{claim:CB-what-wstar-close}. For the primal-dual objective without the AL term
\begin{align*}
\max_{w \geq 0} \min_{v}  L^{\text{CB}}(w, v) 
    &  \;\coloneqq \E_{s, a \sim \mu} \left[ w(s,a)r(s,a) \right] 
    - \E_{s, a \sim \mu}[v(s) (w(s,a) - 1)],
\end{align*}
we have $(w^\star, v^\star) \in \argmax_{w \geq 0} \argmin_{v} L^{\text{CB}}(w,v)$ by strong duality. Moreover, since $w^\star$ is realizable, it satisfies the validity constraint $\E_{a \sim \mu(\cdot|s)}[w^\star(s,a)] = 1$ for all $s$. Therefore, by Lemma \ref{lemma:optimal-solution-with-AL-unchanged} adding the ALM term does not change the optimal solution and we have $(w^\star, v^\star) \in \argmax_{w \geq 0} \argmin_{v} L^{\text{CB}}_{\text{AL}}(w,v)$.

We follow similar steps as in the proof of Lemma \ref{lemma:regularized-unregularized-closeness-MAB} and decompose $L^{\text{CB}}_{\text{AL}}(w^\star, v) -  L^{\text{CB}}_{\text{AL}}(\hat w, v)$ according to 
\begin{align*}
    & L^{\text{CB}}_{\text{AL}}(w^\star, v) -  L^{\text{CB}}_{\text{AL}}(\hat w, v)\\
    & \quad = \underbrace{L^{\text{CB}}_{\text{AL}}(w^\star, v) - L^{\text{CB}}_{\text{AL}}(w^\star, \hat v(w^\star))}_{\coloneqq T_1} + \underbrace{L^{\text{CB}}_{\text{AL}}(w^\star, \hat v(w^\star)) - \hat L^{\text{CB}}_{\text{AL}}(w^\star, \hat v(w^\star))}_{\coloneqq T_2}\\
    & \qquad+ \underbrace{\hat L^{\text{CB}}_{\text{AL}}(w^\star, \hat v(w^\star)) - \hat L^{\text{CB}}_{\text{AL}}(\hat w, \hat v)}_{\coloneqq T_3} + \underbrace{\hat L^{\text{CB}}_{\text{AL}}(\hat w, \hat v) - \hat L^{\text{CB}}_{\text{AL}}(\hat w, v)}_{\coloneqq T_4}\\
    & \qquad + \underbrace{\hat L^{\text{CB}}_{\text{AL}}(\hat w, v) - L^{\text{CB}}_{\text{AL}}(\hat w, v)}_{\coloneqq T_5},
\end{align*}
where $\hat v_w = \arg \min_{v \in \cV} \hat L^{\text{CB}}_{\text{AL}}(w, v).$ Each term is bounded as follows:
\begin{itemize}
    \item $T_1 = 0$ because $w^\star$ satisfies the optimization constraints. 
    \item $T_2 \leq \epsilon^{\text{CB}}_{\text{stat}}$ due to Lemma \ref{lemma:objective-statistical-error-CB}.
    \item $T_3 \leq 0$ because $\hat w = \arg \max_{w \in \cW} \hat  L^{\text{CB}}_{\text{AL}}(\hat v_w, w)$.
    \item $T_4 \leq 0$ because $\hat v = \arg \min_{v \in \cV} \hat L^{\text{CB}}_{\text{AL}}(v, \hat w)$.
    \item $T_5 \leq \epsilon^{\text{CB}}_{\text{stat}}$ due to Lemma \ref{lemma:objective-statistical-error-CB}.
\end{itemize}
Summing up the bounds on each term proves part \ref{claim:CB-what-wstar-close}.
\end{proof}

\begin{proof}[Proof of Lemma \ref{lemma:lowerbound-dhatw}]
We leverage the closeness of the objective at $w^\star$ and $\hat w$ established in Lemma \ref{lemma:objective-statistical-error-CB} to show that the ALM term at $\hat w$ is small. Since $w^\star$ satisfies the validity constraints, the objective at $w^\star$ simplifies to
\begin{align*}
    L^{\text{CB}}_{\text{AL}}(w^\star,v)
    & = \E_{s,a \sim \mu}[r(s,a) w^\star(s,a)] + \underbrace{\E_{s,a \sim \mu}[v(s) (1-w^\star(s,a))]}_{=0} - \underbrace{  \E_{s \sim \mu} [(\E_{a \sim \mu(\cdot|s)}[w(s,a)] - 1)^2]}_{=0}\\
    & = \E_{s,a \sim \mu}[r(s,a) w^\star(s,a)].
\end{align*}
Consider the objective difference at $v(s) = r^\star(s) \coloneqq \max_a r(s,a)$:
\begin{align*}
    & L^{\text{CB}}_{\text{AL}}(w^\star, r^\star) -  L^{\text{CB}}_{\text{AL}}(\hat w, r^\star)\\
    = & \sum_{s} \mu(s) r^\star(s) - \sum_{s,a} \mu(s,a) r(s,a) \hat w(s,a) - \sum_s r^\star(s) \left(\mu(s) - \sum_a \mu(s,a) \hat w(s,a)\right) \\
    & \quad +   \E_{s \sim \mu}\left[\left(\frac{d_{\hat w}}{\mu}(s) -1
\right)^2\right]\\
    = & \sum_{s,a} \mu(s,a) [r^\star(s) - r(s,a)] \hat w(s,a) +   \E_{s \sim \mu}\left[\left(\frac{d_{\hat w}}{\mu}(s) -1
\right)^2\right].
\end{align*}
Since $L^{\text{CB}}_{\text{AL}}(w^\star, v) -  L^{\text{CB}}_{\text{AL}}(\hat w, v) \lesssim \epsilon^{\text{CB}}_{\text{stat}}$ by Lemma \ref{lemma:objective-statistical-error-CB}, we conclude that 
\begin{align*}
    \sum_{s,a} \mu(s,a) [r^\star(s) - r(s,a)] \hat w(s,a) +   \E_{s \sim \mu}\left[\left(\frac{d_{\hat w}}{\mu}(s) -1
\right)^2\right] \lesssim \epsilon^{\text{CB}}_{\text{stat}} 
\end{align*}
Moreover, since the first term is nonnegative due to $\hat w(s,a) \geq 0$ and $r^\star(s) - r(s,a) \geq 0$, both of the terms in the above inequality are bounded by $ \epsilon^{\text{CB}}_{\text{stat}}$ and thereby proving part \ref{claim:avg-over-dhat-small}. 

The above result also allows us to bound the mass on the subset $\cS_{\text{s}}$ that contains the states that violate state occupancy validity
\begin{align*}
    \epsilon^{\text{CB}}_{\text{stat}} \gtrsim \sum_s \mu(s) \left[\left(\frac{d_{\hat w}}{\mu}(s) - 1\right)^2\right]
     \geq \sum_{s \in \cS_{\text{s}}} \mu(s) \left[\left(\frac{d_{\hat w}}{\mu}(s) - 1\right)^2\right]
    \geq \frac{1}{4}\sum_{s \in \cS_{\text{s}}} \mu(s) \gtrsim \sum_{s \in \cS_{\text{s}}} \mu(s),
\end{align*}
where we used the fact that $\frac{d_{\hat w}}{\mu}(s) \leq \frac{1}{2}$ and thus $\left(\frac{d_{\hat w}}{\mu}(s) - 1\right)^2 \geq \frac{1}{4}$ by definition of $\cS_{\text{s}}$. This concludes the proof of part \ref{claim:mass-on-invalid-small}.
\end{proof}

\section{Proofs for MDPs}

In this section, we begin by introducing some additional notation. The original primal-dual objective without ALM term is given by
\begin{align}\label{eq:MDP-population-objective}
        \max_{w \geq 0} \min_{v} L^{\text{MDP}} (w, v) \coloneqq (1-\gamma) \E_{s \sim \rho} [v(s)] + \E_{s,a \sim \mu} \left[w(s,a) e_v(s,a)\right].
\end{align}
Define $w^\star(s,a) = d^{\pi^\star}(s,a)/\mu(s,a)$ and $v^\star(s) = V^*(s)$. By strong duality, one has $(w^\star,v^\star) \in \arg\max_{w \geq 0} \arg\min_{v} L^{\text{MDP}}(w,v)$. Additionally, define $\zeta^\star_{w,u} = \arg \max_{\zeta < 0}  L_{\text{AL}}^{\text{model-free}} (w, v, u, \zeta), \ \forall w \in \cW, u \in \cU$ and $\zeta^\star_w = \zeta^\star_{w, u^\star_w} \ \forall w \in \cW$ where $u^\star_w$ is defined in Theorem \ref{thm:CORAL-suboptimality}. Also, denote $\zeta^\star = \zeta^\star_{w^\star}$ and $u^\star = u^\star_{w^\star} $.

The rest of this section is organized as follows. In Appendix \ref{app:practical-implementation}, we provide some details regarding practical implementation of the offline learning algorithm with ALM. In Appendix \ref{app:model-free-objective}, we derive the objective of model-free CORAL algorithm. Appendix \ref{app:CORAL-suboptimality} contains the proof of performance upper bound on model-based and model-free CORAL algorithms (Theorem \ref{thm:CORAL-suboptimality}), which relies on several lemmas subsequently proved in Appendices \ref{app:u-zeta-bounded} through \ref{app:ratio-not-too-small}.

\subsection{On practical implementations}\label{app:practical-implementation}

In our algorithms for CB and MDP, we need to compute summations of form $\sum_{a \in \cA}$. This can be implemented efficiently when $|\cA|$ is small. When $|\cA|$ is large or even infinite, one can utilize numerical methods to estimate the summation with desired precision.  Additionally, in Algorithm \ref{alg:PAL-MDP-model-based}, we need to evaluate a term $\sum_{s',a'} P(s'|s,a)\pi_w(a'|s')u(s',a')$.
In practice, we can evaluate this term by numerical integration.

\subsection{Derivation of the model-free CORAL objective \eqref{eq:population-objective-model-free-MDP}}\label{app:model-free-objective}

For $f(x) = (x-1)^2$, the Fenchel conjugate $f_*$ is given by 
\begin{align}\label{eq:fenchel-conjugate}
    f_*(x) = \max_y ( xy-f(y) ) = \max_y \left( xy - y^2 + 2y - 1\right) = \left(\frac{x+2}{2}\right)^2 - 1.
\end{align}
Since ${d_w(s)}/{d^{\pi_w}(s)} \geq 0$, we have $x^\star_w(s,a) \geq - 2$ and thus it is sufficient to only consider domain $x(s,a) \geq - 2$, over which $f_*(x)$ is invertible.

Let $g(x) = - f_*^{-1}(x) = 2 - 2\sqrt{x+1}$, which is a convex function on $[-1, +\infty)$. Similar to \citet{nachum2019dualdice}, we use Fenchel duality to estimate $g\left(u(s,a) - \gamma \mathbb{P}^{\pi_w} u(s,a)\right)$. By Fenchel duality, any convex function $g(x)$ can be written as $g(x) = \max_\zeta x \zeta - g_*(\zeta)$. In the case of $g(x)$, the Fenchel conjugate is given by $g_*(x) = - x - 2 - {1}/{x}$ with domain $x < 0$. Therefore, we write
\begin{align*}
    \E_{\mu}[w(s,a)g\left(u(s,a) - \gamma \mathbb{P}^{\pi_w} u(s,a)\right)]
    & = \E_{\mu}[w(s,a)\max_{\zeta < 0}  \left(u(s,a) - \gamma (\mathbb{P}^{\pi_w} u)(s,a) \right) \zeta - g_*(\zeta)] \\
    & = \E_{\mu}[w(s,a)\max_{\zeta < 0}  \left(u(s,a) - \gamma (\mathbb{P}^{\pi_w} u)(s,a) \right) \zeta + \zeta + 1/\zeta + 2].
\end{align*}
The interchangeability principle \cite{rockafellar2009variational, dai2017learning} allows us to convert the inner maximization step over scalar $\zeta$ to an overall maximization over $\zeta: \cS \times \cA \rightarrow \mathbb{R}^-$. Replacing this term in the objective \eqref{eq:AL-variational-form-model-based} results in the model-free objective in \eqref{eq:population-objective-model-free-MDP}. 

\subsection{Proof of Theorem \ref{thm:CORAL-suboptimality}}\label{app:CORAL-suboptimality}

We start by deriving an expression for $x^\star_w$ and characterizing bounds on $u^\star_w$ and $\zeta^\star_{w,v}$ in the following lemma. The proof is presented in Appendix \ref{app:u-zeta-bounded}.

\begin{lemma}\label{lemma:u-zeta-bounded} 
For any $w \in \cW$ and $v \in \cV$, one has $x^\star_w(s,a) = 2{d_w(s)}/{d^{\pi_w}(s)} - 2$, $|u^\star_w(s,a)| \leq \frac{1}{1-\gamma}(B_x^2/4 + B_x)$, and $|\zeta^\star_{w,v}(s,a)| \in \left[\frac{2}{2+B_x}, \frac{2}{2-B_x}\right]$.
\end{lemma}

Bounding the suboptimality of policies returned by both model-based and model-free variants of CORAL follow a similar analysis. We first characterize the statistical error in approximating population objectives by their empirical versions and use it to establish the closeness of $\hat w$ and $w^\star$. The lemma below captures these approximation errors for the model-based objective, whose proof can be found Appendix \ref{app:stat-error-model-based-MDP}.

\begin{lemma}[Empirical and population closeness, model-based CORAL]\label{lemma:stat-error-model-based-MDP} Fix $\delta > 0$ and define 
\begin{align}\label{eq:eps_model_based}
    \epsilon_{\text{stat}}^{\text{model-based}} \coloneqq (B_u + (1+B_v) B_w) \sqrt{\frac{B_u\log(|\cP||\cU||\cW||\cV|/\delta) }{N}}.
\end{align}
For any $w \in \cW, v \in \cV,$ and $u \in \cU$, the following statements hold with probability at least $1-\delta $
\begin{enumerate}[label = (\Roman*)]
    \item\label{claim:empirical-population-obj-model-based} $\left| L_{\text{AL}}^{\text{model-based}}(w, v, u) - \hat L_{\text{AL}}^{\text{model-based}}(w, v, u) \right| \leq \epsilon_{\text{stat}}^{\text{model-based}}$;
    \item\label{claim:wstar-what-model-based} $L_{\text{AL}}^{\text{model-based}}(w^\star, v^\star, u^\star) - L_{\text{AL}}^{\text{model-based}}(\hat w, v^\star, u^\star_{\hat w}) \leq 2 \epsilon_{\text{stat}}^{\text{model-based}}$.
\end{enumerate}
\end{lemma}

In Appendix \ref{app:stat-error-model-free-MDP}, we prove a similar lemma for the model-free objective.
\begin{lemma}[Empirical and population closeness, model-free CORAL]\label{lemma:stat-error-model-free-MDP}
Fix $\delta > 0$ and define 
\begin{align}\label{eq:eps_model_free}
    \epsilon_{\text{stat}}^{\text{model-free}} \coloneqq (B_u + (1+B_v +B_{\zeta}(B_u+1))B_w)\sqrt{\frac{\log(|\cU||\cW||\cV||\cZ|/\delta) }{N}} .
\end{align}
For any $w \in \cW, v \in \cV,$ and $u \in \cU$, the following statements hold with probability at least $1-\delta $
\begin{enumerate}[label = (\Roman*)]
    \item\label{claim:empirical-population-obj-model-free} $\left| L_{\text{AL}}^{\text{model-free}}(w, v, u) - \hat L_{\text{AL}}^{\text{model-free}}(w, v, u) \right| \leq \epsilon_{\text{stat}}^{\text{model-free}}$;
    \item\label{claim:wstar-what-model-free} $L_{\text{AL}}^{\text{model-free}}(w^\star, v^\star, u^\star) - L_{\text{AL}}^{\text{model-free}}(\hat w, v^\star, u^\star_{\hat w}) \leq 2 \epsilon_{\text{stat}}^{\text{model-free}}$.
\end{enumerate}
    
\end{lemma}

The final key lemma demonstrates that in model-based and model-free CORAL, the ALM terms enforce lower bounds on the ratio of the estimated occupancy of learned weights $d_{\hat w}(s)$ and the actual occupancy of the learned policy $d^{\pi_{\hat w}}(s)$ in most states. The proof of this lemma is given in Appendix \ref{app:ratio-not-too-small}.

\begin{lemma}[Occupancy validity by the ALM, MDP]\label{lemma:ratio-not-too-small} 
For $\hat w$ computed by the model-based CORAL
Algorithm \ref{alg:PAL-MDP-model-based}, define the state space subspace $\cS_s \coloneqq \left\{s \ \middle| {d_{\hat w}(s)}  \leq  \frac{1}{2} {d^{\pi_{\hat w}}(s)} \right\}.$ For any fixed $\delta > 0$, the following statements hold with probability at least $1-\delta$
\begin{enumerate}[label=(\Roman*)]
    \item\label{claim:advantage-avg} $\E_{s,a \sim \mu} \left[-A^\star(s,a) \hat w(s,a) \right] \lesssim \epsilon^{\text{model-based}}_{\text{stat}}$;
    \item\label{claim:ratio-not-small} $\sum_{s \in \cS_{\text{s}}} d^{\pi_{\hat w}}(s) \lesssim (1-\gamma)^{-2} \epsilon^{\text{model-based}}_{\text{stat}}$.
\end{enumerate}
Similarly, for $\hat w$ computed by the model-free CORAL
Algorithm \ref{alg:PAL-MDP-model-free}, define the state space subspace $\cS_s \coloneqq \left\{s \ \middle| {d_{\hat w}(s)}  \leq  \frac{1}{2} {d^{\pi_{\hat w}}(s)} \right\}.$ For any fixed $\delta > 0$, the following statements hold with probability at least $1-\delta$
\begin{enumerate}[label=(\Roman*)]
    \item\label{claim:advantage-avg2} $\E_{s,a \sim \mu} \left[-A^\star(s,a) \hat w(s,a) \right] \lesssim \epsilon^{\text{model-free}}_{\text{stat}}$;
    \item\label{claim:ratio-not-small2} $\sum_{s \in \cS_{\text{s}}} d^{\pi_{\hat w}}(s) \lesssim (1-\gamma)^{-2} \epsilon^{\text{model-free}}_{\text{stat}}$.
\end{enumerate}
\end{lemma}

Given the above lemmas, we proceed to prove the suboptimality bounds in terms of statistical errors defined in \eqref{eq:eps_model_based} and \eqref{eq:eps_model_free}. In the rest of this section, we drop the superscripts model-based and model-free from statistical errors to avoid cluttered notation.

In view of the performance difference lemma in~\citet[Lemma 6.1]{kakade2002approximately}, one has
    \begin{align*}
        J(\pi^\star) - J(\hat \pi) & = \mathbb{E}_{s \sim d^{\hat \pi}} \left[ \sum_a A^\star (s,a) \left(\pi^\star(a|s) - \hat \pi(a|s)\right) \right] = \mathbb{E}_{s \sim d^{\hat \pi}} \left[ \sum_a - A^\star (s,a)  \hat \pi(a|s) \right],
    \end{align*}
    where $d^{\hat \pi} = d^{\pi_{\hat w}}$. Here, we used the fact that the expectation of the optimal advantage over optimal policy is zero $\sum_a A^\star (s,a) \pi^\star(a|s) = 0$. Lemma \ref{lemma:ratio-not-too-small} links an expectation of $-A^\star(s,a)$ to the statistical error. With this lemma at hand and using the definition $\cS_{\text{s}} = \{s \ | \  d_{\hat w}(s)  \leq d^{\hat \pi}(s)/2\}$, we continue to decompose and bound the suboptimality
    \begin{align}\notag 
        & \mathbb{E}_{s \sim d^{\hat \pi}} \left[ \sum_a - A^\star (s,a)  \hat \pi(a|s) \right] \\ \notag 
        & \quad = \sum_{s \in \cS_{\text{s}}}  d^{\hat \pi}(s) \left[ \sum_a - A^\star (s,a)  \hat \pi(a|s) \right] + \sum_{s \not \in \cS_{\text{s}}}  d^{\hat \pi}(s) \left[ \sum_a - A^\star (s,a)  \hat \pi(a|s) \right]\\ \label{eq:bound-sum-S_s}
        & \quad \lesssim \frac{1}{(1-\gamma)^3} \epsilon_{\text{stat}} + \sum_{s \not \in \cS_{\text{s}}, d_{\hat w}(s) \neq 0}  \frac{d^{\hat \pi}(s)}{d_{\hat w}(s)} \left[ \sum_a - A^\star (s,a)  {\hat w(s,a) \mu(s,a)} \right]\\ \label{eq:use-S_s-def}
        & \quad \quad \quad + \sum_{s \not \in \cS_{\text{s}}, d_{\hat w}(s) = 0}  d^{\hat \pi}(s) \left[ \sum_a - \frac{1}{|\cA|}A^\star (s,a)   \right] \notag \\
        & \quad \leq \frac{1}{(1-\gamma)^3} \epsilon_{\text{stat}} + 2 \sum_{s \not \in \cS_{\text{s}}}  \left[ \sum_a - A^\star (s,a)  {\hat w(s,a) \mu(s,a)} \right]
    \end{align}
    In \eqref{eq:bound-sum-S_s}, we used part \ref{claim:ratio-not-small} in Lemma \ref{lemma:ratio-not-too-small} and that $-A^\star(s,a) \leq 1/(1-\gamma)$ and in \eqref{eq:use-S_s-def} we used the definition of $\cS_{\text{s}}$ to bound the ratio $d^{\hat \pi}(s)/d_{\hat w}(s)$ by 2 and the fact that $d_{\hat w}(s) = 0$ implies $d^{\hat \pi}(s) = 0$ for $s \notin \cS_{\text{s}}$. We then apply part \ref{claim:advantage-avg} in in Lemma \ref{lemma:ratio-not-too-small} to bound the second term by $\E_{s, a \sim \mu}[-A^\star(s,a) \hat w(s,a)]$ and thus the overall suboptimality:
    \begin{align*}
        J(\pi^\star) - J(\hat \pi) \lesssim \frac{1}{(1-\gamma)^3} \epsilon_{\text{stat}} + \E_{s,a \sim \mu}\left[-A^\star(s,a) \hat w(s,a)\right] \lesssim \frac{1}{(1-\gamma)^3} \epsilon_{\text{stat}}.
    \end{align*}

\subsection{Proof of Lemma \ref{lemma:u-zeta-bounded}}\label{app:u-zeta-bounded}

\paragraph{Derivation of $x^\star_w$.} Recall from Appendix \ref{app:model-free-objective} that for $f(x) = (x-1)^2$, the Fenchel conjugate is $f_*(x) = \left(\frac{x+2}{2}\right)^2 - 1$. Therefore, for any $(s,a)$, 
\begin{align}\notag 
    & x^\star_w(s,a) = \arg \max_x \left(d_w(s) x - d^{\pi_w}(s)\left( \left(\frac{x+2}{2}\right)^2 - 1\right) \right) = 2\frac{d_w(s)}{d^{\pi_w}(s)} - 2\\ \notag 
    \Rightarrow \; & \tilde x_w(s,a) = \clip\left(2\frac{d_{w}(s)}{d^{\pi_{w}}(s)} - 2, -B_x, B_x\right).
\end{align}

\paragraph{Bound on $u^\star_w$.} 
Recall that $u^\star_w$ is defined as the fixed point of the following Bellman-like equation
\begin{align}
    u(s,a) = f_*(\tilde x_w(s,a)) + \gamma (\mathbb{P}^{\pi_w} u)(s,a).
\end{align}
The above equation has a solution since $f_*(\tilde x_w(s,a))$ is bounded
\begin{align*}
   \left(\frac{2-B_x}{2}\right)^2 -1 \leq f_*(\tilde x_w(s,a)) \leq \left(\frac{B_x+2}{2}\right)^2 - 1.
\end{align*}
One can view $u^\star_w$ as the Q-function of policy $\pi_w$ with the reward function $f_*(\tilde x_w(s,a))$, which leads to $|u^\star_w(s,a)| \leq\frac{1}{1-\gamma} \max\left\{  1 - \left( \frac{2-B_x}{2}\right)^2 , \left(\frac{B_x+2}{2}\right)^2 - 1 \right\} = \frac{1}{1-\gamma}(B_x^2/4 + B_x)$.

\paragraph{Bound on $\zeta^\star_{w, u}$.} To see the bound on $\zeta^\star_{w, u}$, recall that by definition, 
\begin{align}
    \zeta^\star_{w,u} = \arg\max_{\zeta < 0}  \E_{(s,a,s') \sim \mu, a' \sim \pi_w(\cdot|s')} [ w(s,a) \left( (u(s,a) - \gamma u(s',a')+1)\zeta(s,a)  + 1/\zeta(s,a) \right) ].
\end{align}
It is easy to show that $|\zeta^\star_{w,u}(s,a)| = (u(s,a) - \gamma(\mathbb{P}^{\pi_w} u)(s',a')+1)^{-1/2} = (f_*(\tilde x_w(s,a))+1)^{-1/2}$. Since $\tilde x_w(s,a) \in [-B_x, B_x]$, we have  $|\zeta^\star_{w,u}(s,a)| \in \left[\frac{2}{2+B_x}, \frac{2}{2-B_x}\right]$.

\subsection{Proof of Lemma \ref{lemma:stat-error-model-based-MDP}}\label{app:stat-error-model-based-MDP}

\subsubsection{Proof of part \ref{claim:empirical-population-obj-model-based}}\label{app:model-based-empirical-population-claim}
We decompose the difference between the population and empirical objective into three terms $L_{\text{AL}}^{\text{model-based}} - \hat L_{\text{AL}}^{\text{model-based}} = T_1 + T_2 + T_3$ defined as follows
\begin{align*}
    T_1 & \coloneqq (1-\gamma) \E_{\rho} \left[v(s) + \sum_a u(s,a) \pi_w(a|s)\right] - (1-\gamma) \frac{1}{N_0} \sum_{i=1}^{N_0} \left( v(s_i)  +\sum_a u(s_i,a) \pi_w(a|s_i)\right)\\
    T_2 & \coloneqq \E_{\mu} \left[w(s,a)(r(s,a) + \gamma \sum_{s'}P(s'|s,a) v(s') - v(s))\right] - \frac{1}{N} \sum_{i=1}^N w(s_i, a_i) \left[r_i + \gamma v(s_i') - v(s_i) \right]\\
    T_3 &: = \E_{\mu} \left[w(s,a) \left(f_*^{-1} \left(u(s,a) - \gamma P^{\pi_w} u(s,a)\right) \right)\right] - \frac{1}{N} \sum_{i=1}^N w(s_i, a_i) \left[f_*^{-1}\left(u(s_i, a_i) - \gamma \hat P^{\pi_w} u(s_i,a_i)\right)  \right]
\end{align*}
We subsequently show that the absolute values of the above error terms satisfy the following high probability upper bounds:
\begin{subequations}
\begin{align}\label{eq:T1bound-model-based}
    |T_1| & \lesssim (B_v + B_u) \sqrt{\frac{\log |\cV||\cU|/\delta}{N_0}},\\\label{eq:T2bound-model-based}
    |T_2| & \lesssim (1+B_v)B_w \sqrt{\frac{\log|\cV||\cW|/\delta}{N}},\\\label{eq:T3bound-model-based}
    |T_3| & \lesssim  B_w \sqrt{\frac{B_u\log|\cP||\cU||\cW|/\delta }{N}}
\end{align}
\end{subequations}
Taking $N_0 = N$ and noting that $B_w \geq 1$ due to realizability of $w^\star$ yield that
\begin{align*}
    \left| L_{\text{AL}}^{\text{model-based}} - \hat L_{\text{AL}}^{\text{model-based}} \right| & \lesssim (B_v + B_u)\sqrt{\frac{\log |\cV||\cU|/\delta}{N_0}} + (1+B_v)B_w \sqrt{\frac{B_u\log(|\cP||\cU||\cW||\cV|/\delta)}{N}}\\
    & \lesssim \epsilon_{\text{stat}}^{\text{model-based}}. 
\end{align*}

\paragraph{Proof of bound \eqref{eq:T1bound-model-based} on $|T_1|$.} Since $|v(s)| \leq B_v, |u(s,a)| \leq B_u$ for all $v \in \cV$ and $u \in \cU$ and $s_i$ are independent, we can apply Hoeffding's inequality and union bound to conclude the advertised bound \eqref{eq:T1bound-model-based} on $|T_1|$.

\paragraph{Proof of the bound \eqref{eq:T2bound-model-based} on $|T_2|$.} By boundedness of $w, v$, we have 
\begin{align*}
    \left|w(s,a) (r(s,a) + \gamma v(s') - v(s)) \right| \leq B_w (1+ (\gamma+1) B_v) \leq B_w (1+\gamma)(1+B_v). 
\end{align*}
As before, due to boundedness and independence of variables $w(s_i, a_i) [r_i + \gamma v(s_i') - v(s_i)]$, Hoeffding's inequality can be applied, giving the bound \eqref{eq:T2bound-model-based} on $|T_2|$.

\paragraph{Proof of the bound \eqref{eq:T3bound-model-based} on $|T_3|$.} We decompose $T_3 = T_{3,1} + T_{3,2}$, where $T_{3,1}$ and $T_{3,2}$ are defined as 
\begin{align*}
    T_{3,1} & \coloneqq \E_{\mu} \left[w(s,a) \left(f_*^{-1} \left(u(s,a) - \gamma ({\mathbb{P}}^{\pi_w} u)(s,a)\right) \right)\right] - \E_{\mu} \left[w(s,a) \left(f_*^{-1} \left(u(s,a) - \gamma (\mathbb{\hat P}^{\pi_w} u)(s,a)\right) \right)\right]\\
    T_{3,2} & \coloneqq \E_{\mu} \left[w(s,a) \left(f_*^{-1} \left(u(s,a) - \gamma (\mathbb{\hat P}^{\pi_w} u)(s,a)\right) \right)\right]\\
    & \qquad - \frac{1}{N} \sum_{i=1}^N w(s_i, a_i) \left[f_*^{-1}\left(u(s_i, a_i) - \gamma (\mathbb{\hat P}^{\pi_w} u)(s_i,a_i)\right)  \right]
\end{align*}
Recall that $f_*^{-1}(x) = 2 \sqrt{x+1} - 2$ from Appendix \ref{app:model-free-objective}. The absolute value of $T_{3,2}$ can be immediately bounded using Hoeffding's inequality:
\begin{align}\label{eq:T_32-bound-model-based}
    |T_{3,2}| \lesssim B_w \sqrt{\frac{B_u\log |\cW||\cU|\delta}{N}}.
\end{align}

To bound $|T_{3,1}|$, we first use the inequality given in Lemma \ref{lemma:square-root-inequality}, setting $b_i, x_i, y_i$ for each $(s,a)$ according to
\begin{align*}
    b_i & = \begin{cases}
    1+u(s,a) \quad & i = 0\\
    \gamma \sum_{a'} \pi_w(a'|s') u(s'|a') \quad & 1 \leq i \leq |\cS|
    \end{cases}\\
    x_i & = \begin{cases}
    1 \quad & i = 0\\
    P(s'|s,a)\quad & 1 \leq i \leq |\cS|
    \end{cases}, \quad 
    y_i = \begin{cases}
    1 \quad & i = 0\\
    P(s'|s,a)\quad & 1 \leq i \leq |\cS|
    \end{cases}
\end{align*}
Thus by Lemma \ref{lemma:square-root-inequality}, we obtain the following bound on $T_{3,1}^2$
\begin{align}\notag 
    T^2_{3,1} & = \bigg( \E_{\mu} \left[w(s,a) \left(f_*^{-1} \left(u(s,a) - \gamma \mathbb{P}^{\pi_w} u(s,a)\right) \right)\right]  - \E_{\mu} \left[w(s,a) \left(f_*^{-1} \left(u(s,a) - \gamma \mathbb{\hat P}^{\pi_w} u(s,a)\right) \right)\right] \bigg)^2\\\notag 
    & \lesssim B_w \vast( \E_\mu \left[ \sqrt{1+ u(s,a) - \gamma \sum_{s'} P(s'|s,a) \sum_{a'} \pi_w(a'|s') u(s',a')} \right]\\\notag 
    & \quad - \E_\mu \left[ \sqrt{1+ u(s,a) - \gamma \sum_{s'} \hat P(s'|s,a) \sum_{a'} \pi_w(a'|s') u(s',a')} \right] \vast)^2\\\label{eq:initial-bound-T3}
    & \leq B^2_w B_u \E_\mu \left[ \sum_{s'} \left( \sqrt{P(s'|s,a)} - \sqrt{\hat P(s'|s,a)}\right)^2\right].
\end{align} 
Note that the terms under square root are always nonnegative because for any transition $P$
\begin{align*}
    1+ u(s,a) - \gamma \sum_{s'}  P(s'|s,a) \sum_{a'} \pi_w(a'|s') u(s',a') \geq 1 - B_u - \gamma B_u \geq 1 - 2B_u \geq 0.
\end{align*}
Then, we use the concentration result on maximum likelihood model estimation stated in Theorem \ref{lemma:model-estimation-error-MDP} and a union bound on $w \in \cW$ and $v \in \cV$ to conclude that
\begin{align}\label{eq:T_31-bound-model-based}
    |T_{3,1}| \lesssim B_w \sqrt{\frac{B_u\log|\cP||\cU||\cW|/\delta }{N}}.
\end{align}

\subsubsection{Proof of part \ref{claim:wstar-what-model-based}}

To prove the second part, let $\hat v_w$ and $\hat u_w$ denote the solutions to the model-based empirical objective
\begin{align*}
    \hat v_w, \hat u_w = \argmin_{v \in \cV} \argmin_{u \in \cU} \hat L_{\text{AL}}^{\text{model-based}} (w, v, u)
\end{align*}
Decompose the objective difference according to
\begin{align*}
    & L_{\text{AL}}^{\text{model-based}}(w^\star, v^\star, u^\star) - L_{\text{AL}}^{\text{model-based}}(\hat w, v^\star, u^\star_{\hat w})\\
    & \quad = {L_{\text{AL}}^{\text{model-based}}(w^\star, v^\star, u^\star) - L_{\text{AL}}^{\text{model-based}}(w^\star, \hat v_{w^\star}, \hat u_{w^\star})} { \quad \coloneqq T_1}\\
    & \qquad + L_{\text{AL}}^{\text{model-based}}(w^\star, \hat v_{w^\star}, \hat u_{w^\star}) - \hat L_{\text{AL}}^{\text{model-based}}(w^\star, \hat v_{w^\star}, \hat u_{w^\star}) \quad \coloneqq T_2\\
    & \qquad + \hat L_{\text{AL}}^{\text{model-based}}(w^\star, \hat v_{w^\star}, \hat u_{w^\star}) - \hat L_{\text{AL}}^{\text{model-based}}(\hat w, \hat v_{\hat w}, \hat u_{\hat w}) \quad \coloneqq T_3\\
    & \qquad + \hat L_{\text{AL}}^{\text{model-based}}(\hat w, \hat v_{\hat w}, \hat u_{\hat w}) - \hat L_{\text{AL}}^{\text{model-based}}(\hat w, v^\star, u^\star_{\hat w}) \quad \coloneqq T_4\\
    & \qquad + \hat L_{\text{AL}}^{\text{model-based}}(\hat w, v^\star, u^\star_{\hat w}) - L_{\text{AL}}^{\text{model-based}}(\hat w, v^\star, u^\star_{\hat w}) \quad \coloneqq T_5
\end{align*}
We bound each term:
\begin{itemize}
    \item $T_1 \leq 0$ because $v^\star, u^\star = \arg \min_{v} \arg \min_{u} L_{\text{AL}}^{\text{model-based}} (w^\star, v, u)$;
    \item $T_2 \leq \epsilon^{\text{model-based}}_{\text{stat}}$ by Lemma \ref{lemma:stat-error-model-based-MDP};
    \item $T_3 \leq 0$ because $\hat w = \arg \max_{w \in \cW} \hat L_{\text{AL}}^{\text{model-based}} (w, \hat v_{w}, \hat u_{w})$;
    \item $T_4 \leq 0$ because $\hat v_{w}, \hat u_{w} = \arg \min_{v \in \cV} \arg \min_{u \in \cU} \hat L_{\text{AL}}^{\text{model-based}} (w, v, u)$;
    \item $T_5 \leq \epsilon^{\text{model-based}}_{\text{stat}}$ by Lemma \ref{lemma:stat-error-model-based-MDP}.
\end{itemize}

\subsection{Proof of Lemma \ref{lemma:stat-error-model-free-MDP}}\label{app:stat-error-model-free-MDP}

\subsubsection{Proof of part \ref{claim:empirical-population-obj-model-free}}
We decompose the difference $L_{\text{AL}}^{\text{model-free}} - \hat L_{\text{AL}}^{\text{model-free}} = T_1 + T_2 + T_3$ into three error terms
\begin{align*}
    T_1 & \coloneqq (1-\gamma) \E_{\rho} \left[v(s) + \sum_a u(s,a) \pi_w(a|s)\right] - (1-\gamma) \frac{1}{N_0} \sum_{i=1}^{N_0} \left( v(s_i)  +\sum_a u(s_i,a) \pi_w(a|s_i)\right)\\
    T_2 & \coloneqq \E_{\mu} \left[w(s,a)(r(s,a) + \gamma \sum_{s'}P(s'|s,a) v(s') - v(s))\right] - \frac{1}{N} \sum_{i=1}^N w(s_i, a_i) \left[r_i + \gamma v(s_i') - v(s_i) \right]\\
    T_3 &: =  \E_{(s,a,s') \sim \mu, a' \sim \pi_w(\cdot|s')} [ w(s,a) \left((u(s,a) - \gamma u(s',a'))\zeta(s,a) -g_\star(\zeta(s,a) ) \right) ] \\
    & \qquad - \frac{1}{N} \sum_{i=1}^N w(s_i, a_i)\left[\left(u(s_i,a_i) - \gamma \sum_{a' \in \cA} u(s'_i, a') \pi_w(a'|s'_i )\right)\zeta(s_i,a_i) -g_\star(\zeta(s_i,a_i) )  \right].
\end{align*}
The absolute values of the error terms above satisfy the following upper bounds with high probability
\begin{subequations}
    \begin{align}\label{eq:T1bound-model-free}
        |T_1| & \lesssim (B_v + B_u) \sqrt{\frac{\log (|\cV||\cU|/\delta)}{N_0}},\\ \label{eq:T2bound-model-free}
        |T_2| & \lesssim (1+B_v)B_w \sqrt{\frac{\log (|\cV||\cW|/\delta)}{N}},\\ \label{eq:T3bound-model-free}
        |T_3| & \lesssim (1+B_{\zeta}(B_u+1))B_w \sqrt{\frac{\log|\cU||\cW||\cZ|/\delta}{N}}.
    \end{align}
\end{subequations}
The bounds on the first two error terms $|T_1|$ and $|T_2|$ are already shown in Appendix \ref{app:model-based-empirical-population-claim}. To bound $|T_3|$, recall that $g_\star(x) = -x - 2 - \frac{1}{x}, \ \forall x < 0$. Also, $|\zeta(s,a)| \in (B_{\zeta,L}, B_{\zeta,U})$ for any $\zeta \in \cZ$ and any $(s,a)$, and $B_{\zeta} \triangleq \max\{ B_{\zeta, U}, B_{\zeta, L}^{-1} \}$. Therefore, the individual error terms in $|T_3|$ satisfy the following bound  
\begin{align*}
\left|w(s,a) \left((u(s,a) - \gamma u(s',a'))\zeta(s,a) -g_\star(\zeta(s,a) ) \right) \right| \leq& B_w ((1+\gamma)B_u B_{\zeta,U} + B_{\zeta,U} + B_{\zeta,L}^{-1} + 2 ).
\end{align*}
Thus, by Hoeffding's inequality and a union bound on $\cW, \cU,$ and $\cZ$, we obtain the upper bound \eqref{eq:T3bound-model-free} on $|T_3|$. Summing up the bounds given in \eqref{eq:T1bound-model-free}, \eqref{eq:T2bound-model-free}, and \eqref{eq:T3bound-model-free} and noting that $B_w \geq 1$ due to realizability of $w^\star$, we obtain
\begin{align*}
    & L_{\text{AL}}^{\text{model-free}}(w, v, u, \zeta) - \hat L_{\text{AL}}^{\text{model-free}}(w, v, u, \zeta) \\
    & \lesssim (B_v + B_u)\sqrt{\frac{\log |\cV||\cU|/\delta}{N_0}} + (1+B_v +B_{\zeta}(B_u+1))B_w \sqrt{\frac{\log|\cU||\cW||\cV||\cZ|/\delta }{N}}\\
    & \lesssim \epsilon_{\text{stat}}^{\text{model-free}}.
\end{align*}

\subsubsection{Proof of part \ref{claim:wstar-what-model-free}} 
Define the following solutions to the empirical model-free objective
\begin{align*}
    \hat v_w, \hat u_w, \hat \zeta_w & = \arg \min_{v \in \cV} \arg \min_{u \in \cU} \arg \max_{\zeta \in \cZ} \hat L_{\text{AL}}^{\text{model-free}} (w, v, u, \zeta), \quad \forall w \in \cW\\
    \hat \zeta(w, u) & = \arg \max_{\zeta \in \cZ} \hat L_{\text{AL}}^{\text{model-free}} (w, v, u, \zeta) \quad \forall w \in \cW, u \in \cU
\end{align*}
Decompose the objective difference according to
\begin{align*}
    & L_{\text{AL}}^{\text{model-free}}(w^\star, v^\star, u^\star, \zeta^\star) - L_{\text{AL}}^{\text{model-free}}(\hat w, v^\star, u^\star_{\hat w}, \zeta^\star_{\hat w})
    \\
    & \quad = {L_{\text{AL}}^{\text{model-free}}(w^\star, v^\star, u^\star, \zeta^\star) - L_{\text{AL}}^{\text{model-free}}(w^\star, \hat v_{w^\star}, \hat u_{w^\star}, \zeta^\star(w^\star, \hat u_{w^\star}))} { \quad \coloneqq T_1}\\
    & \qquad + L_{\text{AL}}^{\text{model-free}}(w^\star, \hat v_{w^\star}, \hat u_{w^\star}, \zeta^\star_{w^\star, \hat u_{w^\star}}) - \hat L_{\text{AL}}^{\text{model-free}}(w^\star, \hat v_{w^\star}, \hat u_{w^\star}, \zeta^\star_{w^\star, \hat u_{w^\star}}) { \quad \coloneqq T_2}\\
    & \qquad + \hat L_{\text{AL}}^{\text{model-free}}(w^\star, \hat v_{w^\star}, \hat u_{w^\star}, \zeta^\star_{w^\star, \hat u_{w^\star}}) - \hat L_{\text{AL}}^{\text{model-free}}(w^\star, \hat v_{w^\star}, \hat u_{w^\star}, \hat \zeta_{w^\star}) { \quad \coloneqq T_3}\\
    & \qquad + \hat L_{\text{AL}}^{\text{model-free}}(w^\star, \hat v_{w^\star}, \hat u_{w^\star}, \hat \zeta_{w^\star})  -\hat L_{\text{AL}}^{\text{model-free}}(\hat w, \hat v_{\hat w}, \hat u_{\hat w}, \hat \zeta_{\hat w}) \quad \coloneqq T_4\\
    & \qquad + \hat L_{\text{AL}}^{\text{model-free}}(\hat w, \hat v_{\hat w}, \hat u_{\hat w}, \hat \zeta_{\hat w}) - \hat L_{\text{AL}}^{\text{model-free}}(\hat w, v^\star, u^\star_{\hat w}, \hat \zeta_{\hat w, u^\star_{\hat w}}) \quad \coloneqq T_5\\
    & \qquad +  \hat L_{\text{AL}}^{\text{model-free}}(\hat w, v^\star, u^\star_{\hat w}, \hat \zeta_{\hat w, u^\star_{\hat w}}) - L_{\text{AL}}^{\text{model-free}}(\hat w, v^\star, u^\star_{\hat w}, \hat \zeta_{\hat w, u^\star_{\hat w}}) \quad \coloneqq T_6\\
    & \qquad +  L_{\text{AL}}^{\text{model-free}}(\hat w, v^\star, u^\star_{\hat w}, \hat \zeta_{\hat w, u^\star_{\hat w}}) - L_{\text{AL}}^{\text{model-free}}(\hat w, v^\star, u^\star_{\hat w}, \zeta^\star_{\hat w}) \quad \coloneqq T_7
\end{align*}
We bound each term:
\begin{itemize}
    \item $T_1 \leq 0$ because $v^\star, u^\star = \arg \min_{v, u} L_{\text{AL}}^{\text{model-free}} (w^\star, v, u, \zeta^\star(w^\star, u))$;
    \item $T_2 \leq \epsilon_{\text{model-free}}$ by part \ref{claim:empirical-population-obj-model-free};
    \item $T_3 \leq 0$ because  $\hat \zeta_{w^\star} = \hat \zeta_{w^\star, \hat u_{w^\star}} = \arg \max_{\zeta \in \cZ} \hat L_{\text{AL}}^{\text{model-free}} (w^\star, \hat v_{w^\star}, \hat u_{w^\star}, \zeta)$;
    \item $T_4 \leq 0$ because $\hat w = \arg \max_{w \in \cW} \hat L_{\text{AL}}^{\text{model-free}} (w, \hat v_w, \hat u_w, \hat \zeta_w)$;
    \item $T_5 \leq 0$ because $\hat v_{\hat w}, \hat u_{\hat w} = \arg \min_{v \in \cV, u \in \cU} \hat L_{\text{AL}}^{\text{model-free}} (\hat w, v, u, \hat \zeta_{\hat w, u)}$;
    \item $T_6 \leq \epsilon_{\text{model-free}}$ by part \ref{claim:empirical-population-obj-model-free};
    \item $T_7 \leq 0$ because $\zeta^\star_{\hat w} = \zeta^\star_{\hat w,  u^\star_{\hat w}} = \arg \max_{\zeta < 0} L_{\text{AL}}^{\text{model-free}} (\hat w, v^\star, u^\star_{\hat w}, \zeta)$.
\end{itemize}

\subsection{Proof of Lemma \ref{lemma:ratio-not-too-small}}\label{app:ratio-not-too-small}

We provide proof only for the model-based algorithm and let $\hat w = \hat w^{\text{model-based}}$ for notation convenience. The proof for a model-free algorithm follows analogously, noting the fact that $L^{\text{model-free}}_{AL}(w, v^\star, u^\star_w, \zeta^\star_w) = L^{\text{model-based}}_{AL}(w, v^\star, u^\star_w)$
and we can replace Lemma \ref{lemma:stat-error-model-based-MDP} with Lemma \ref{lemma:stat-error-model-free-MDP} to prove the model-free version.

\subsubsection{Proof of part \ref{claim:advantage-avg}}

Consider the expression of the model-based objective $L_{\text{AL}}^{\text{model-based}} (w^\star, v^\star, u^\star)$ at the optimal solution where $u^\star \coloneqq u^\star_{w^\star}$:
\begin{align}\notag 
    & L_{\text{AL}}^{\text{model-based}} (w^\star, v^\star, u^\star)\\ \notag 
    = & (1-\gamma) \E_{s \sim \rho} [v^\star(s)] + \E_{s,a \sim \mu} \left[w^\star(s,a) e_{v^\star}(s,a)\right] - \E_{s \sim d^{\pi_{w^\star}}} \left(\frac{d_{w^\star}(s)}{d^{\pi_{w^\star}}(s)} - 1\right)^2\\ \label{eq:expressing-L-at-optimal-model-based-with-advantage} 
    = & (1-\gamma) \E_{s \sim \rho} [V^\star(s)] + \E_{s,a \sim \mu} \left[w^\star(s,a)A^\star(s,a)\right]
\end{align}
The first equation comes from the fact that $u^\star$ is the optimal solution to the variational lower bound, making it equal to the $f$-divergence. To see this, recall from Lemma \ref{lemma:u-zeta-bounded} that $x^\star_w (s,a) = 2 d_w(s)/d^{\pi_w}(s) - 2$ and $\tilde x_w(s,a) = \clip (x^\star_w(s,a), -B_x, B_x)$. Since $x^\star_{w^\star}(s,a) = 0$, we have $\tilde x_{w^\star}(s,a) = x^\star_{w^\star}(s,a)$ and thus $u^\star$ recovers the $f$-divergence.

In Equation \eqref{eq:expressing-L-at-optimal-model-based-with-advantage}, we wrote $v^\star(s) = V^\star(s)$ since $v^\star(s)$ is the optimal solution to the primal-dual program without the ALM term and is equal to the optimal value function \cite{zhan2022offline}. We also used the fact that $e_{v^\star}(s,a) = r(s,a) + \gamma \sum_{s'} P(s'|s,a) v^\star(s') - v^\star(s) = A^{\star}(s,a)$ is the optimal advantage function, and that ${d_{w^\star}(s)} = {d^{\pi_{w^\star}}(s)}$ by definition and realizability of $w^\star$. Moreover, the second term in \eqref{eq:expressing-L-at-optimal-model-based-with-advantage} is zero since it captures the optimal advantage of optimal policy. Therefore, we conclude that
\begin{align}
    L_{\text{AL}}^{\text{model-based}} (w^\star, v^\star, u^\star) = (1-\gamma) \E_{s \sim \rho} [V^\star(s)].
\end{align}
Given the above expression of the objective at $(w^\star, v^\star, u^\star)$, we write the following objective difference 
\begin{align}\notag 
    & L_{\text{AL}}^{\text{model-based}}(w^\star, v^\star, u^\star) - L_{\text{AL}}^{\text{model-based}}(\hat w, v^\star, u^\star_{\hat w})\\\notag 
    = & (1-\gamma) \E_{s \sim \rho} [V^\star(s)] - (1-\gamma) \E_{s \sim \rho} [v^\star(s)] - \E_{s,a \sim \mu} \left[\hat w(s,a) e_{v^\star}(s,a)\right]\\ \label{eq:difference-u-part-model-based}
    & \quad + (1-\gamma) \E_{s \sim \rho, a \sim \pi_{\hat w}} [u^\star_{\hat w}(s, a)] + \E_\mu \left[ \hat w(s,a) f_*^{-1}\left(u^\star_{\hat w}(s,a) - \gamma  (\mathbb{P}^{\pi_{\hat w}} u^\star_{\hat w})(s,a)\right)\right]\\\notag 
    = & - \E_{s,a \sim \mu} \left[\hat w(s,a) A^\star(s,a)\right] - \E_{d^{\pi_{\hat w}}} [f_*(\tilde x_{\hat w}(s,a))] + \E_{d_{\hat w}} [\tilde x_{\hat w}(s,a)]
\end{align}
The last line uses $e_{v^\star}(s,a) = A^\star(s,a)$ as well as the definition of $u^\star_{\hat w}$ as the fixed point solution to 
\begin{align}\notag 
    u^\star_{\hat w}(s,a) \coloneqq f_*(\tilde x_{\hat w}(s,a)) + \gamma (\mathbb{P}^{\pi_{\hat w}} u^\star_{\hat w})(s,a),
\end{align}
which allows us to write \eqref{eq:difference-u-part-model-based} in the original $f$-divergence variational form \eqref{eq:f-divergence-variational-form} with $\tilde x_{\hat w}$ as variable. Lemma \ref{lemma:stat-error-model-based-MDP} asserts that $L_{\text{AL}}^{\text{model-based}}(w^\star, v^\star, u^\star) - L_{\text{AL}}^{\text{model-based}}(\hat w, v^\star, u^\star_{\hat w}) \lesssim \epsilon^{\text{model-based}}_{\text{stat}}$. Therefore, 
\begin{align}\label{eq:claim1-model-based}
    {- \E_{s,a \sim \mu} \left[\hat w(s,a) A^\star(s,a)\right]} - {\E_{d^{\pi_{\hat w}}} [f_*(\tilde x_{\hat w}(s,a))] + \E_{d_{\hat w}} [\tilde x_{\hat w}(s,a)]} \lesssim \epsilon^{\text{model-based}}_{\text{stat}}.
\end{align}
We next argue that both terms in inequality above are nonnegative and conclude that 
\begin{subequations}
    \begin{align}\label{eq:expected_advantage_bound}
        {- \E_{s,a \sim \mu} \left[\hat w(s,a) A^\star(s,a)\right]} & \lesssim \epsilon^{\text{model-based}}_{\text{stat}}\\ \label{eq:variational-form-bound}
        - {\E_{d^{\pi_{\hat w}}} [f_*(\tilde x_{\hat w}(s,a))] + \E_{d_{\hat w}} [\tilde x_{\hat w}(s,a)]} & \lesssim \epsilon^{\text{model-based}}_{\text{stat}}
    \end{align}
\end{subequations}
The first term is nonnegative because for the optimal advantage function we have $A^\star(s,a) \leq 0$ for all $s \in \cS$ and $a \in \cA$. We write the second term as
\begin{align*}
    - {\E_{d^{\pi_{\hat w}}} [f_*(\tilde x_{\hat w}(s,a))] + \E_{d_{\hat w}} [\tilde x_{\hat w}(s,a)]}& = \E_{d^{\pi_{\hat w}}} \left[ \frac{d_{\hat w}(s)}{d^{\pi_{\hat w}}(s)} \tilde x_{\hat w}(s,a)  -  f_*(\tilde x_{\hat w}(s,a)) \right].
\end{align*}
We then show that each term inside the expectation is nonnegative:
\begin{align}\label{eq:variational-form-nonnegative}
    \frac{d_{w}(s)}{d^{\pi_{w}}(s)} \tilde x_w(s,a)  -  f_*(\tilde x_{w}(s,a)) \geq 0 \quad \forall s \in \cS, w \in \cW.
\end{align}

\paragraph{Proof of bound \eqref{eq:variational-form-nonnegative}.} we separate the argument into three cases and use the expression of $\tilde x_w$ given in Lemma \ref{lemma:u-zeta-bounded}. 
\begin{enumerate}
    \item {When $1-B_x/2 \leq \frac{d_{w}(s)}{d^{\pi_w}(s)}\leq B_x/2 + 1$}, we have $\tilde x_w (s,a) = \left(2\frac{d_w(s)}{d^{\pi_w}(s)} - 2\right)$ and therefore
    \begin{align*}
        \frac{d_{  w}(s)}{d^{\pi_{  w}}(s)} \tilde x_{  w}(s,a)  -  f_*(\tilde x_{  w}(s,a)) = \left( \frac{d_{  w}(s)}{d^{\pi_{  w}}(s)} - 1\right)^2 \geq 0.
    \end{align*}
    \item When $\frac{d_{w}(s)}{d^{\pi_w}(s)}> B_x/2 + 1$, substitute $\tilde x_w(s,a) = B_x$ to arrive at
    \begin{align*}
        \frac{d_{w}(s)}{d^{\pi_w}(s)} B_x - \left( \left(\frac{B_x}{2}+1\right)^2 - 1\right) \geq \left(\frac{B_x}{2}+1\right) B_x - \frac{B_x^2}{4} - B_x = \frac{B_x^2}{4} \geq 0.
    \end{align*}
    \item\label{case:case3-variational-form} Similarly,  when $\frac{d_{w}(s)}{d^{\pi_w}(s)} < 1 - B_x/2$, substitute $\tilde x_w(s,a) = -B_x$ to arrive at
    \begin{align}\label{eq:case3-variational-form}
        -\frac{d_{w}(s)}{d^{\pi_w}(s)} B_x - \left( \left(1-\frac{B_x}{2}\right)^2 - 1\right) \geq \left(\frac{B_x}{2}-1\right) B_x - \frac{B_x^2}{4} + B_x = \frac{B_x^2}{4} \geq 0.
    \end{align}
\end{enumerate}

\subsubsection{Proof of part \ref{claim:ratio-not-small}}
 We derive the second part by using the bound \eqref{eq:variational-form-bound} restricted on the set $\cS_s$. When $s \in \cS_s$, we have $\frac{d_{\hat w}(s)}{d^{\pi_{\hat w}}(s)} \leq \frac{1}{2}$ and thus the variational form falls into the case \ref{case:case3-variational-form} in the proof of bound \eqref{eq:variational-form-nonnegative}. Therefore, for $s \in \cS_s$, we have $\tilde x_{\hat w}(s,a) = - B_x$ and 
\begin{align}\label{eq:lowerbound-fdivergence-model-based}
    \frac{d_{\hat w}(s)}{d^{\pi_{\hat w}}(s)} \tilde x_{\hat w}(s,a)  -  f_*(\tilde x_{\hat w}(s,a))  \gtrsim {(1-\gamma)^2} \quad \forall s \in \cS_{\text{s}}.
\end{align}
We use the bound in \eqref{eq:variational-form-bound} as well as \eqref{eq:lowerbound-fdivergence-model-based} to conclude that 
\begin{align*}
    \epsilon^{\text{model-based}}_{\text{stat}} \gtrsim & \E_{d_{\hat w}} [\tilde x_{\hat w}(s,a)] -  \E_{d^{\pi_{\hat w}}} [f_*(\tilde x_{\hat w}(s,a))]\\
    = & \sum_{s} d^{\pi_{\hat w}}(s) \left(\frac{d_{\hat w}(s)}{d^{\pi_{\hat w}}(s)} \tilde x_{\hat w}(s,a)  -  f_*(\tilde x_{\hat w}(s,a)) \right)\\
    \gtrsim &\sum_{s \in \cS_{\text{s}}} {(1-\gamma)^2}  d^{\pi_{\hat w}}(s),
\end{align*}
which leads to the second advertised claim $\sum_{s \in \cS_{\text{s}}} d^{\pi_{\hat w}}(s) \lesssim (1-\gamma)^{-2} \epsilon_{\text{stat}}$.

\section{Auxiliary results}

\begin{theorem}[Convergence of MLE for learning transitions \cite{van2000empirical}]\label{lemma:model-estimation-error-MDP}
    Given a realizable model class $\cP = \{P: (\cS, \cA) \rightarrow \Delta(\cS)\}$ that contains the true model $P^\star$ and a dataset $\cD_m = \{(s_i, a_i, s'_i)\}$ with $(s_i, a_i) \overset{\mathrm{iid}}{\sim} \mu, s'_i \sim P^\star(\cdot |s_i, a_i)$, let $\hat P$ be
\begin{align*}
    \hat P = \arg \max_{P \in \cP} \sum_{i=1}^N \ln P(s'_i|s_i, a_i).
\end{align*}
Fix the failure probability $\delta > 0$. Then, with probability at least $1-\delta$, we have the following concentration on the squared Hellinger distance between $\hat P$ and $P^\star$:
\begin{align*}
    \E_{s,a \sim \mu} \left[\sum_{s'} \left(\sqrt{\hat P(s'|s,a)} - \sqrt{P^\star(s'|s,a)}\right)^2 \right]\lesssim \frac{\log(|\cP|/\delta) }{N}.
\end{align*}
\end{theorem}

\begin{lemma}\label{lemma:square-root-inequality} For any $0 \leq b_i \leq B$ and $x_i, y_i \geq 0$ for $i \in \{0, \dots, n\}$, the following holds
\begin{align}\label{eq:square-root-inequality}
\left(\sqrt{\sum_{i=0}^n b_i x_i} - \sqrt{\sum_{i=0}^n b_i y_i} \right)^2 \leq B \sum_{i=0}^n \left(\sqrt{x_i} - \sqrt{y_i} \right)^2.
\end{align}
    
\end{lemma}

\begin{proof}
We expend the left-hand side of \eqref{eq:square-root-inequality}, use Cauchy-Schwarz inequality, and then complete the square:
\begin{align*}
     \left(\sqrt{\sum_{i=1}^n b_i x_i} - \sqrt{\sum_{i=1}^n b_i y_i} \right)^2
    = & \sum_i b_i x_i + \sum_i b_i y_i - 2 \sqrt{\left( \sum_i b_i x_i\right) \left( \sum_i b_i y_i\right)}\\
    \leq & \sum_i b_i x_i + \sum_i b_i y_i - 2 \sum_i b_i \sqrt{x_i y_i}\\
    = & \sum_i (\sqrt{b_i x_i} - \sqrt{b_i y_i})^2\leq  B \sum_i \left(\sqrt{x_i} - \sqrt{y_i} \right)^2.
\end{align*}
\end{proof}

\begin{lemma}
\label{lemma:optimal-solution-with-AL-unchanged}
    For any two arbitrary sets $\cX, \cY$, let $f(\cdot, \cdot): \cX \times \cY \to \mathbb{R}$ be an arbitrary function. Let $\cX_0 = \{ x \in \cX \ | \ \inf_{y \in \cY} f(x,y) > -\infty \}$ and assume $\cX_0$ is non-empty. For any $x \in \cX_0$,  assume there exists $y^*(x) \in \cY$ s.t. $f(x,y^*(x)) = \min_{y \in \cY} f(x,y)$.
    Also, let $\cX^*_{\text{p}} = \{ x \in \cX_0 \ | \ x \in \arg \max_{x \in \cX_0} 
    f(x,y^\star(x))\}$ and assume $\cX^*_{\text{p}}$ is non-empty. 
    For a nonnegative function $A(\cdot)$ on $\cX$, let $\cX^* = \{ x \in \cX^*_{\text{p}} \ | \ A(x) = 0 \}$ and assume $\cX^*$ is non-empty. Define $f^{AL}(x,y) = f(x,y) - A(x)$. Then 
    \begin{align*}
        x \in \cX_0 \Longleftrightarrow \inf_{y \in \cY} f^{AL}(x,y) > -\infty.
    \end{align*}
    and for any $x \in \cX_0$, 
    \begin{align*}
        x \in \cX^* \Longleftrightarrow x \in \arg\max_{x \in \cX_0} \min_{y \in \cY} f^{AL}(x,y).
    \end{align*}
\end{lemma}

\begin{proof}
Note that for any fixed $x$, $f^{AL}(x,y)$ is a constant shift of $f(x,y)$, which implies that $\inf_{y \in \cY} f(x,y) > -\infty \Longleftrightarrow \inf_{y \in \cY} f^{AL}(x,y) > -\infty$. This also implies that for any $x \in \cX_0$,
$\arg \min_{y \in \cY} f(x, y) =  \arg \min_{y \in \cY} f^{AL}(x, y)$.

For any $x \in \cX_0$, let $y^*(x)$ denote any one of $y \in \cY$ s.t. $f(x, y^*(x)) = \min_{y \in \cY} f(x, y)$. 

Now for any $x^* \in \cX^*$, we have 
\begin{align*}
     &f(x^*, y^*(x^*)) \geq f(x, y^*(x)), \  \forall x \in \cX_0.
    \\
    \Longrightarrow & f(x^*, y^*(x^*)) - A(x^*) \geq f(x, y^*(x)) - A(x), \  \forall x \in \cX_0.
    \\
    \Longrightarrow & f^{AL}(x^*, y^*(x^*))\geq f^{AL}(x, y^*(x)), \  \forall x \in \cX_0.
    \\
    \Longrightarrow & \min_{y \in \cY} f^{AL}(x^*, y)\geq \min_{y \in \cY} f^{AL}(x, y), \  \forall x \in \cX_0.
    \\
     \Longrightarrow & x^* \in \arg\max_{x \in \cX_0} \min_{y \in \cY} f^{AL}(x,y).
\end{align*}

For the other direction, given any $x_0 \in \arg\max_{x \in \cX_0} \min_{y \in \cY} f^{AL}(x,y)$, we have 
\begin{align}
\label{eq:def_optimal_solution_of_f_AL}
\begin{split}
     & \min_{y \in \cY} f^{AL}(x_0, y)\geq \min_{y \in \cY} f^{AL}(x, y), \  \forall x \in \cX_0. \\
     \Longrightarrow & f^{AL}(x_0, y^*(x_0))  \geq  f^{AL}(x, y^*(x)), \  \forall x \in \cX_0.
\end{split}
\end{align}
Fix any $x^*  \in \cX^* \subseteq \cX_0$, we have $f(x^*, y^*(x^*)) \geq f(x_0, y^*(x_0))$ and $-A(x^*) \geq -A(x_0)$ by definition. Now assume $x_0 \notin \cX^*$. Then either $f(x^*, y^*(x^*)) > f(x_0, y^*(x_0))$ if $x_0 \notin \cX^*_{\text{p}}$, or $-A(x^*) > -A(x_0)$ if $x_0 \in \cX^*_{\text{p}}\backslash \cX^*$. Either one of the above two conditions implies that 
\begin{align*}
    f(x^*, y^*(x^*)) - A(x^*) > f(x_0, y^*(x_0)) - A(x_0) \Longrightarrow f^{AL}(x^*, y^*(x^*)) > f^{AL}(x_0, y^*(x_0)),
\end{align*}
which contradicts with \eqref{eq:def_optimal_solution_of_f_AL}. Therefore, $x_0 \in \cX^*$.
\end{proof}

\end{document}